\documentclass{article}

\usepackage{microtype}
\usepackage{graphicx}
\usepackage{subfigure}
\usepackage{booktabs}

\usepackage{hyperref}

\usepackage[preprint]{icml2024}

\usepackage{amsmath}
\usepackage{amssymb}
\usepackage{mathtools}
\usepackage{amsthm}

\usepackage[capitalize,noabbrev]{cleveref}
\usepackage{thm-restate}

\theoremstyle{plain}
\newtheorem{thm}{Theorem}

\theoremstyle{definition}

\theoremstyle{remark}

\usepackage[textsize=tiny]{todonotes}

\usepackage{amsfonts,amscd,mathrsfs}

\usepackage{multirow}

\icmltitlerunning{
A Universal Approximation Theorem for Nonlinear Resistive Networks
}

\begin{document}

\twocolumn[
\icmltitle{
A Universal Approximation Theorem for Nonlinear Resistive Networks
}


\begin{icmlauthorlist}
\icmlauthor{Benjamin Scellier}{rain}
\icmlauthor{Siddhartha Mishra}{eth}
\end{icmlauthorlist}

\icmlaffiliation{rain}{Rain AI, San Francisco, USA}
\icmlaffiliation{eth}{SAM, D-Math and ETH AI Center, ETH Zurich, Switzerland}

\icmlcorrespondingauthor{Benjamin Scellier}{benjamin@rain.ai}

\icmlkeywords{universal function approximation,resistor network,nonlinear resistive network,deep resistive network,equilibrium propagation,self-learning machine}

\vskip 0.3in
]



\printAffiliationsAndNotice{} 

\begin{abstract}
Resistor networks have recently been studied as analog computing platforms for machine learning, particularly due to their compatibility with the Equilibrium Propagation training framework. In this work, we explore the computational capabilities of these networks. We prove that electrical networks consisting of voltage sources, linear resistors, diodes, and voltage-controlled voltage sources (VCVSs) can approximate any continuous function to arbitrary precision. Central to our proof is a method for translating a neural network with rectified linear units into an approximately equivalent electrical network comprising these four elements. Our proof relies on two assumptions: (a) that circuit elements are ideal, and (b) that variable resistor conductances and VCVS amplification factors can take any value (arbitrarily small or large). Our findings provide insights that could guide the development of universal self-learning electrical networks.
\end{abstract}

\section{Introduction}

Machine learning (ML) has become an integral part of modern technology, with neural networks playing a central role in many ML applications. A fundamental characteristic of neural networks is their ability to implement or approximate arbitrary functions: they are \emph{universal function approximators}. However, training and deploying neural networks on graphics processing units (GPUs) is energy-intensive, primarily due to the separation of computation and memory in GPUs, and the digital nature of computing in these processors. As the need for energy-efficient ML systems becomes pressing, researchers explore alternative physical substrates based on analog physics as a replacement for GPU-based neural networks \citep{momeni2024training}. One key area of research has been the development of analog electrical versions of neural networks, which aim to implement their arithmetic operations -- including matrix-vector multiplications and nonlinear activation functions -- and the backpropagation algorithm in the analog domain. However, these systems are typically mixed-signal (analog/digital), executing the nonlinear activation functions in the digital domain, and therefore necessitating the use of analog-to-digital converters and digital-to-analog converters \citep{xiao2020analog}. This hurdle has motivated researchers to develop and explore alternative training frameworks, potentially better suited for training analog physical systems \citep{scellier2017equilibrium,scellier2022agnostic,wright2022deep,lopez2023self,wanjura2024fully,momeni2023backpropagation} (for recent reviews, see also \cite{momeni2024training,stern2023learning}). However, the physical systems compatible with these training frameworks are typically not isomorphic to neural networks. In particular, unlike neural networks, the computational capabilities and expressivity of these systems remain largely unexplored. Universality has been studied in optical systems \citep{marcucci2020theory} (see also \cite{larocque2021universal} where universal linear optical transformations are studied), and a recent study has shown how neural ordinary differential equations can be implemented in analog electrical circuits \citep{gao2023kirchhoffnet}.

We investigate the question of universal function approximation within the context of resistor networks. These electrical networks are composed of variable resistors that serve as trainable weights and voltage sources representing input variables. The steady state of such electrical networks is characterized by the principle of least power dissipation. Although resistor networks have been studied as platforms for analog computing since the 1980s \citep{harris1989resistive}, their applications were initially limited to inference tasks. More recently, the equilibrium propagation (EP) framework \citep{scellier2017equilibrium,kendall2020training} has provided a method to train these networks. EP uses local contrastive learning rules to adjust the resistor conductances to perform gradient descent on a cost function \citep{kendall2020training,anisetti2024frequency}. Experimental implementations of resistor networks that self-train using local learning rules closely related to EP have also been realized \citep{dillavou2022demonstration,dillavou2024machine}, demonstrating the feasibility of this approach. Additionally, an experimental demonstration on a memristor crossbar array \citep{yi2023activity} suggests potential energy savings of up to four orders of magnitude compared to neural networks trained on GPUs. Importantly, EP is also applicable in nonlinear networks, without requiring detailed knowledge of the nonlinear elements' characteristics (see Appendix~\ref{sec:equilibrium-propagation} for details). These and other studies have renewed interest in resistor networks \citep{wycoff2022desynchronous,stern2022physical,stern2024training,stern2024physical,kiraz2022impacts,watfa2023energy,oh2023memristor,oh2025defect,guzman2024microscopic}. However, the computational expressivity of these networks has remained unexplored.

In this paper, we derive sufficient conditions for resistor networks to be universal function approximators. A recent line of work has shown that nonlinear computation can be achieved even using purely linear physics \citep{yildirim2024nonlinear,xia2024nonlinear,wanjura2024fully} ; in an analog electrical circuit, this can be achieved by viewing tunable resistances as inputs (see \cite{wanjura2024fully} where the method is briefly outlined). Here, however, we consider networks in which inputs are supplied via voltage sources, tunable resistors serve as adjustable parameters, and the voltage response is read at a set of output branches. Clearly, a network of this kind composed solely of ohmic resistors and voltage sources can only implement linear functions. In such networks, achieving nonlinear computation - which is a prerequisite for any sort of universal computation - requires nonlinear elements. Furthermore, in a network with only voltage sources and passive (linear or nonlinear) resistors, the voltage across any branch is bounded by the sum of the voltages from the sources. Following Kendall \emph{et al.} \cite{kendall2020training}, we address these two issues using diodes for nonlinearities, and voltage-controlled voltage sources (VCVSs) for amplification, allowing the implementation of functions with output voltages greater than input voltages. We prove that, under suitable assumptions, these four elements - ohmic resistors, voltage sources, diodes and VCVSs - are sufficient for universal function approximation. The assumptions are that (a) the elements are ideal, and (b) the conductances of variable resistors and the amplification factors of the VCVSs can take arbitrary values (arbitrarily small or arbitrarily large). To prove this result, we consider a specific nonlinear resistive network architecture similar to the one proposed by Kendall \emph{et al.} \cite{kendall2020training} which we refer to as a `deep resistive network' (DRN), and we show that DRNs can approximate neural networks with rectified linear units (ReLU) to arbitrary accuracy. Since ReLU neural networks are known to be universal function approximators \citep{leshno1993multilayer,yarotsky2017error}, it follows that DRNs are universal approximators, too. Our work thus presents the first universal approximation result for a class of analog electrical networks compatible with EP. Additionally, our proof offers a systematic method for translating a ReLU neural network into an equivalent-sized approximate DRN.

We emphasize that, in our construction, DRNs do not \textit{implement} ReLU neural networks but rather \textit{approximate} them, in contrast to the line of research on analog electrical neural networks that use backpropagation for training. Furthermore, although in this work we assume ideal elements to simplify the mathematical analysis of DRNs, EP does not require these assumptions. In fact, EP is agnostic to the characteristics of untrainable components such as diodes (see Appendix~\ref{sec:equilibrium-propagation} for details).

\begin{figure*}[h!]
\begin{center}
\includegraphics[width=0.6\textwidth]{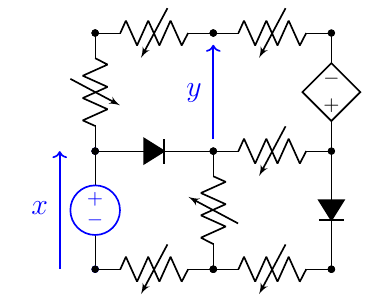}
\end{center}
\caption{
\textbf{A nonlinear resistive network} composed of variable resistors, voltage sources, diodes and voltage-controlled voltage sources (VCVSs). Voltage sources are used as inputs ($x$) and voltages across pairs of `output nodes' are used as outputs ($y$). Variable resistors represent the trainable weights, diodes introduce nonlinearities, and the VCVSs are used for amplification. Such electrical networks are universal function approximators (Theorem~\ref{thm:nrn-universal}). Note: for simplicity, the input terminals of the VCVS are not represented on the figure ; only its output terminals are displayed.
\label{fig:nonlinear-resistive-network}
}
\end{figure*}

\begin{figure*}[h!]
\begin{center}
\includegraphics[width=0.98\textwidth]{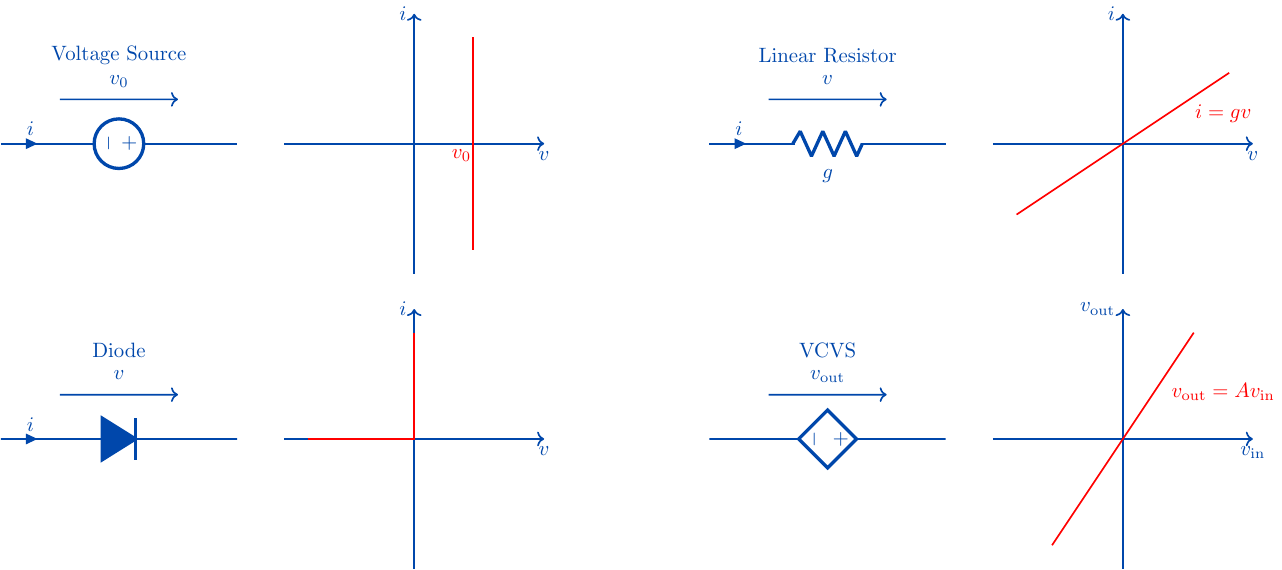}
\end{center}
\caption{
\textbf{Current-voltage (i-v) characteristics of ideal elements.} A linear resistor is characterized by Ohm's law: $i = g v$, where $g$ is the conductance of the resistor ($g=1/r$ where $r$ is the resistance). An ideal diode is characterized by $i=0$ for $v \leq 0$ ("off state", behaving like an open switch), and $v=0$ for $i>0$ ("on state", behaving like a closed switch). An ideal voltage source is characterized by $v=v_0$ for some constant voltage $v_0$. Finally, a voltage-controlled voltage source (VCVS) has four terminals -- two terminals that form an input voltage $v_{\rm in}$ and two terminals that form an output voltage $v_{\rm out}$ -- and is characterized by the relationship $v_{\rm out} = A \; v_{\rm in}$, where $A>0$ is the amplification factor, or gain. Note: for simplicity, the input terminals of the VCVS are not represented on the figure ; only its output terminals are represented.
}
\label{fig:circuit-elements}
\end{figure*}

\section{Nonlinear Resistive Networks}
\label{sec:nonlinear-resistive-network}

We consider an electrical circuit composed of voltage sources, linear resistors, diodes and VCVSs (Figure~\ref{fig:nonlinear-resistive-network}). We call such a circuit a `nonlinear resistive network'. It implements an input-output function as follows: a) a set of voltage sources serve as `input variables' where voltages play the role of input values, and b) a set of branches serve as `output variables' where voltage drops play the role of the model's prediction. Computing with this electrical network proceeds as follows: 1) set the input voltage sources to input values, 2) let the electrical network reach steady state (the DC operating point of the circuit), 3) read the output voltages across output branches.

The steady state of the network is the configuration of branch voltages and branch currents that satisfies all the branch equations, as well as Kirchhoff's current law (KCL) at every node, and Kirchhoff's voltage law (KVL) in every loop. We note that some conditions on the network topology and branch characteristics must be met to ensure that there exists a steady state: for instance, if the network contains a loop formed of voltage sources, then the voltage values must add up to zero in order to obey KVL. Subsequently, we will consider circuits that do not contain loops of voltage sources.

We assume that the above four circuit elements are \textit{ideal}, meaning that their behavior is determined by the following characteristics (see Figure~\ref{fig:circuit-elements}):
\begin{itemize}
\item The current-voltage ($i$-$v$) characteristic of a voltage source satisfies $v=v_0$ for some constant voltage value $v_0$, regardless of the current $i$.
\item A resistor follows Ohm's law, that is, its $i$-$v$ characteristic is $i = g v$, where $g$ is the conductance ($g=1/r$ where $r$ is the resistance).
\item The $i$-$v$ characteristic of a diode satisfies $i=0$ for $v < 0$, and $v=0$ for $i>0$.
\item A VCVS has four terminals: two input terminals whose voltage we denote by $v_{\rm in}$, and two output terminals whose voltage we denote by $v_{\rm out}$. It satisfies the relationship $v_{\rm out} = A v_{\rm in}$, where $A>0$ is the `gain' (or `amplification factor').
\end{itemize}
Under these assumptions, nonlinear resistive networks are universal function approximators in the following sense.

\begin{restatable}[A universal approximation theorem for nonlinear resistive networks]{thm}{nonlinearresistivenetwork}
\label{thm:nrn-universal}
Given any continuous function $f : \mathbb{R}^q \to \mathbb{R}^r$, and given any compact subset $C \subset \mathbb{R}^q$ and any $\epsilon>0$, there exists a nonlinear resistive network with $q$ input voltage sources and $r$ output branches such that, under the above assumptions of ideality, the function $F : \mathbb{R}^q \to \mathbb{R}^r$ that the nonlinear resistive network implements (i.e., the mapping between the inputs and outputs to the network) satisfies
\begin{equation}
\| F(x) - f(x) \| \leq \epsilon, \qquad \forall x \in C.
\end{equation}
\end{restatable}

The main aim of the next section is to prove Theorem~\ref{thm:nrn-universal}.

\section{Universal Approximation Through Neural Network Approximation}
\label{sec:sketch-proof}

To prove Theorem~\ref{thm:nrn-universal}, we show that nonlinear resistive networks can approximate ReLU neural networks to any desired accuracy. Since ReLU neural networks are universal approximators, it follows that nonlinear resistive networks are universal approximators, too.

We proceed in three steps. First, we recall the definition of a ReLU neural network (Section~\ref{sec:deep-neural-network}) and their universal approximation property (Theorem~\ref{thm:relu-nn-universal}). Then we present a layered nonlinear resistive network model similar to the one proposed by Kendall \emph{et al.} \citep{kendall2020training} that we call a `deep resistive network' (Section~\ref{sec:deep-resistive-network}). Under the assumptions of Section~\ref{sec:nonlinear-resistive-network}, we derive the equations characterizing the steady state of a DRN (Lemma~\ref{lma:deep-resistive-network}). Finally, using Lemma~\ref{lma:deep-resistive-network}, we prove that any ReLU neural network can be approximated by a DRN (Theorem~\ref{thm:drn-approximate-relu-nn} in Section~\ref{sec:universal-approximators}). Combining everything, our main Theorem~\ref{thm:nrn-universal} is a consequence of Theorem~\ref{thm:relu-nn-universal} and Theorem~\ref{thm:drn-approximate-relu-nn}. Proofs of the lemmas and theorems are provided in Appendix~\ref{sec:proofs}.

\subsection{ReLU Neural Network}
\label{sec:deep-neural-network}

A neural network processes an input vector $x$ through several transformation stages to produce an output vector. We consider specifically the multilayer perceptron architecture with the rectified linear unit (ReLU) activation function. We refer to this model as the `ReLU neural network'.

A ReLU neural network consists of $L$ stages of transformation, called `layers'. The input vector is processed sequentially from the input layer (of index $\ell = 0$) to the output layer (of index $\ell = L$), through the hidden layers ($1 \leq \ell \leq L-1$). Each layer $\ell$ has a dimension $M_{\ell}$, also known as the number of `units' of the layer. Denoting by $s_k^{(\ell)}$ the state of the $k$-th unit of layer $\ell$, the equations defining the network's state are as follows:

\begin{enumerate}
\item For the input layer:
\begin{equation}
\label{eq:relu-nn-input}
s_k^{(0)} := x_k, \qquad 1 \leq k \leq M_0,
\end{equation}
where $x=(x_1, \ldots, x_{M_0})$ is the input vector.
\item For the hidden layers:
\begin{equation}
s_k^{(\ell)} := \max \left( 0, \sum_{j=1}^{M_{\ell-1}} w_{jk}^{(\ell)} s_j^{(\ell-1)} + b_k^{(\ell)} \right),
\end{equation}
for every $(\ell,k)$ such that $1 \leq \ell \leq L-1$ and $1 \leq k \leq M_\ell$.
\item For the output layer:
\begin{equation}
\label{eq:relu-nn-output}
s_k^{(L)} := \sum_{j=1}^{M_{L-1}} w_{jk}^{(L)} s_j^{(L-1)} + b_k^{(L)}, \qquad 1 \leq k \leq M_L.
\end{equation}
\end{enumerate}

In these equations, the units’ transformations are parameterized by the `weights' $w_{jk}^{(\ell)}$ and `biases' $b_k^{(\ell)}$. The $\max(0, \cdot)$ function is called the `ReLU' nonlinear activation function \citep{glorot2011deep}. The network thus implements a function $G : \mathbb{R}^{M_0} \to \mathbb{R}^{M_L}$ mapping input $x$ to output $s^{(L)}$.

\begin{thm}[ReLU neural networks are universal function approximators]
\label{thm:relu-nn-universal}
ReLU neural networks can approximate any continuous function on a compact subset. Specifically, for any continuous function $f : \mathbb{R}^q \to \mathbb{R}^r$, compact subset $C \subset \mathbb{R}^q$, and $\epsilon > 0$, there exists a ReLU neural network with $L=2$ layers, $M_0=q$ input units, and $M_2=r$ output units such that:
\begin{equation}
\| G(x) - f(x) \| \leq \epsilon, \qquad \forall x \in C,
\end{equation}
where $G$ is the function implemented by the neural network.
\end{thm}

This well-known result was first proved by \cite{leshno1993multilayer} (see also \cite{yarotsky2017error}). We will use this result to prove our Theorem~\ref{thm:nrn-universal}.

\subsection{Deep Resistive Network}
\label{sec:deep-resistive-network}

Next, we present a nonlinear resistive network model similar to that of Kendall \emph{et al.}  \citep{kendall2020training} which takes inspiration from the layered architecture of neural networks. We call this a `deep resistive network'. Our goal will be to prove that DRNs can approximate ReLU neural networks to arbitrary accuracy.

In a DRN, the circuit elements - voltage sources, resistors and diodes - are assembled to form a layered network. A DRN is depicted in Figure \ref{fig:deep-resistive-network} and defined as follows. First of all, we choose a reference node called `ground'. We denote by $L$ the number of layers in the DRN, and for each $\ell$ such that $0 \leq \ell \leq L$, we denote by $N_\ell$ the number of nodes in layer $\ell$.
We choose $N_0=2q+2$ input nodes and $N_L=r$ output nodes.
We also call each node a `unit' by analogy with the units of a neural network. We denote by $v_k^{(\ell)}$ the electrical potential of the $k$-th node of layer $\ell$, which we may think of as the unit's activation. Pairs of nodes from two consecutive layers are interconnected by (variable) resistors (i.e. the `trainable weights'). We denote by $g_{jk}^{(\ell)}$ the conductance of the resistor between the $j$-th node of layer $\ell-1$ and the $k$-th node of layer $\ell$. The layer of index $\ell=0$ is the `input layer', whose units are connected to ground by VCVSs with the same gain $A^{(0)} > 0$. Moreover, there are $q$ voltage sources that play the role of input variables and serve as input voltages to the $2q$ VCVSs. Given an input vector $x=(x_1,\ldots,x_q)$, we set the $q$ input voltage sources to $x_1,\ldots,x_q$, so that the input nodes (VCVS output voltages) have electrical potentials
\begin{equation}
v_{2k}^{(0)} := +A^{(0)} x_k, \qquad v_{2k+1}^{(0)} := - A^{(0)} x_k, \qquad 1 \leq k \leq q.
\end{equation}
The units of intermediate (`hidden') layers ($1 \leq \ell \leq L-1$) of the network are nonlinear: for each hidden unit, a diode is placed between the unit's node and ground, which can be oriented in either of the two directions. For the units of even index $k$, the diode points from the unit's node to ground, so the unit's electrical potential is non-negative: when it reaches $0$, current flows through the diode from ground to the unit's node so as to maintain the unit's electrical potential at zero ; we call such a unit an \textit{excitatory unit}. Conversely, for the units of odd index $k$, the diode points from ground to the unit's node, so the unit's electrical potential is non-positive: when it reaches $0$, the extra current sinks to ground through the diode ; we call this an \textit{inhibitory unit}. Thus, half of the hidden units are excitatory and the other half are inhibitory. Output units are linear, meaning that no diode is used for the output units. Finally, each hidden and output unit $v_{k}^{(\ell)}$ is also linked to ground by a resistor whose conductance is denoted by $g_{k}^{(\ell)}$. In addition, each of the input and hidden layers contains a pair of `bias units' that are connected to ground by voltage sources of fixed voltages:
\begin{equation}
v_{0}^{(\ell)}:=+A^{(\ell)}, \qquad v_{1}^{(\ell)}:=-A^{(\ell)}, \qquad 0 \leq \ell \leq L-1,
\end{equation}
where $A^{(\ell)} > 0$ a layerwise scalar -- the `gain' of layer $\ell$.

Under the assumptions of ideality of the elements (Section~\ref{sec:nonlinear-resistive-network}), the steady state of the DRN - the configuration of branch voltages and branch currents that satisfies all the branch equations, as well as KCL and KVL - is determined by the following set of equations.

\begin{restatable}[Equations of a DRN]{lma}{equationsdrn}
\label{lma:deep-resistive-network}
The electrical potentials of hidden units satisfy:
\begin{equation}
\label{eq:hidden-drn}
v_k^{(\ell)} = 
\left\{
\begin{array}{l}
\displaystyle \max \left( 0, p_k^{(\ell)} \right) \quad \text{if k is even (excitatory unit)}, \\
\displaystyle \min \left( 0, p_k^{(\ell)} \right) \quad \text{if k is odd (inhibitory unit)},
\end{array}
\right.
\end{equation}
where
\begin{equation}
p_k^{(\ell)} = \frac{\sum_{j=0}^{N_{\ell-1}-1} g_{jk}^{(\ell)} v_j^{(\ell-1)} + \sum_{j=2}^{N_{\ell+1}-1} g_{kj}^{(\ell+1)} v_j^{(\ell+1)}}{g_k^{(\ell)} + \sum_{j=0}^{N_{\ell-1}-1} g_{jk}^{(\ell)} + \sum_{j=2}^{N_{\ell+1}-1} g_{kj}^{(\ell+1)}}.
\end{equation}
As for the output units' electrical potentials,
\begin{equation}
v_k^{(L)} = 
\frac{\sum_{j=0}^{N_{L-1}-1} g_{jk}^{(L)} v_j^{(L-1)}}{g_k^{(L)} + \sum_{j=0}^{N_{\ell-1}} g_{jk}^{(L)}}. 
\end{equation}
\end{restatable}

We prove Lemma~\ref{lma:deep-resistive-network} in Appendix~\ref{sec:proof:deep-resistive-network}. A few remarks are in order. First, the equations characterizing the steady state of a DRN share similarities with those of a ReLU neural network (Eq.~\eqref{eq:relu-nn-input}-\eqref{eq:relu-nn-output}): these equations involve a weighted sum of the neighboring units' activities (electrical potentials), and the function $\max(0,\cdot)$ is the `ReLU' nonlinear activation function.

However, the equations of a DRN and a ReLU neural network are not identical: a DRN is not a ReLU neural network. One difference is that, in a DRN, the units of a given layer $\ell$ receive signals not just from the previous layer $\ell-1$ but also from the next layer $\ell+1$ ; in other words, signals flow in both the forward and backward directions. 

Another difference is that these signals are normalized by the sum of the weights (total conductance) $g_k^{(\ell)} + \sum_{j=0}^{N_{\ell-1}-1} g_{jk}^{(\ell)} + \sum_{j=2}^{N_{\ell+1}-1} g_{kj}^{(\ell+1)}$. This results in the following \emph{layerwise maximum principle}, summarized in the following lemma.

\begin{restatable}[A maximum principle for DRNs]{lma}{maximumprinciple}
\label{lma:maximum-principle}
For every $\ell$ such that $0 \leq \ell \leq L$, denote the maximum amplitude of the node electrical potentials in layer $\ell$ by
\begin{equation}
v_{\rm max}^{(\ell)} := \max_{2 \leq k \leq N_\ell-1} \left| v_k^{(\ell)} \right|.
\end{equation}
Then we have
\begin{equation}
\label{eq:to-prove}
v_{\rm max}^{(\ell)} \leq \, \max \left( v_{\rm max}^{(\ell-1)}, A^{(\ell-1)}, \ldots, A^{(L)} \right), \qquad 1 \leq \ell \leq L.
\end{equation}
In particular, the electrical potentials of hidden and output units are bounded by those of input and bias units, that is, for every $k$ and $\ell$,
\begin{equation}
\left| v_k^{(\ell)} \right| \leq \max \left( \max_{2 \leq j \leq N_0-1} \left| v_j^{(0)} \right|, \; \max_{0 \leq \ell \leq L-1} \left| A^{(\ell)} \right| \right).
\end{equation}
\end{restatable}

We prove Lemma~\ref{lma:maximum-principle} in Appendix~\ref{sec:proof:lma:maximum-principle}. As a consequence of this lemma, the signal strengths tend to decrease as the index of the layer ($\ell$) increases. This is the reason why bias units and VCVSs are used: the VCVSs amplify the input signals by a gain $A^{(0)}$ to compensate for the decay in amplitude and ensure that the signals at the output layer do not vanish.

Another difference between DRNs and neural networks is that the `weights' of a DRN are non-negative - a conductance is non-negative - meaning $g_{jk}^{(\ell)} \geq 0$. This is the reason why, in addition to excitatory units, we also use inhibitory units, so the network can communicate `negative signals' too. This is also the reason why we have doubled the number of input VCVSs in the input layer (two VCVSs for each input voltage source).

\begin{figure*}[ht!]
\begin{center}
\fbox{
\includegraphics[width=0.9\textwidth]{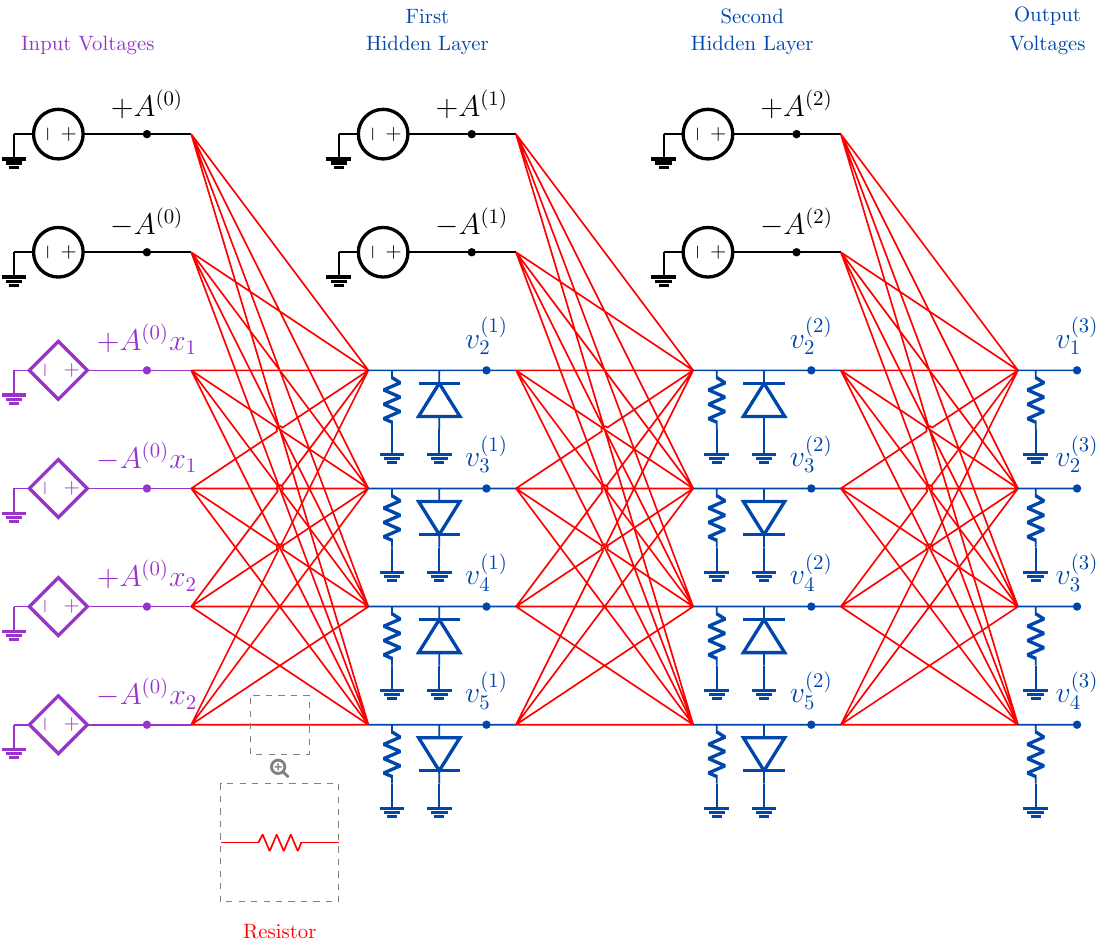}
}

\vspace{0.5cm}

\includegraphics[width=0.9\textwidth]{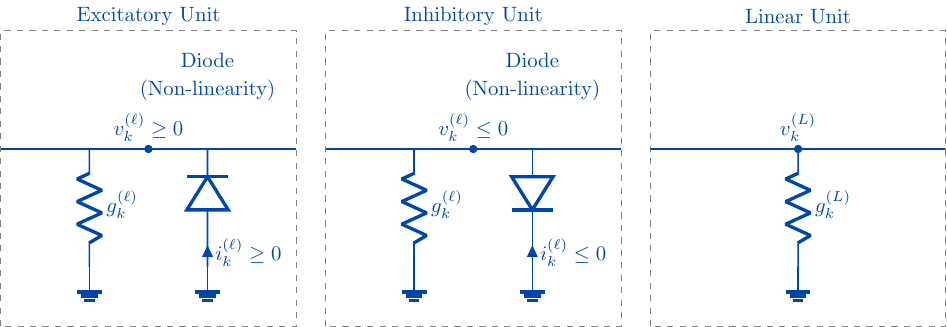}
\end{center}
\caption{
\textbf{A deep resistive network (DRN).} 
\textbf{(Top)} A DRN consists of `units' (excitatory, inhibitory or linear, shown in blue) interconnected by variable resistors (in red). Input voltages ($x_1$ and $x_2$) are applied to the network via voltage sources (not shown in the figure for readability) and are amplified using VCVSs to generate the electrical potentials $+A^{(0)}x_1$, $-A^{(0)}x_1$, $+A^{(0)}x_2$, and $-A^{(0)}x_2$ at designated `input nodes' (in purple). The model's prediction is read from a set of output nodes ($v_1^{(3)}$, $v_2^{(3)}$, $v_3^{(3)}$, and $v_4^{(3)}$). Additionally, each layer includes two voltage sources serving as `biases' (shown in black). \textbf{(Bottom)} Each unit is connected to ground through a resistor. Nonlinear units are constructed by placing a diode between the unit's node and ground. Depending on the orientation of the diode, nonlinear units can be either `excitatory' or `inhibitory'. Units without a diode are called `linear' units.
}
\label{fig:deep-resistive-network}
\end{figure*}

\subsection{Approximating a ReLU Neural Network With a DRN}
\label{sec:universal-approximators}

Given a ReLU neural network, we show that we can build a DRN, parametrized by a number $\gamma>0$, that approximates the neural network up to $O(\gamma)$. To this end, let $L$ be the number of layers of the neural network, and $M_\ell$ the number of units in layer $\ell$, for $0 \leq \ell \leq L$. We denote by $w_{jk}^{(\ell)}$ and $b_k^{(\ell)}$ the weights and biases. We also denote by $s_k^{(\ell)}$ the state of the $k$-th unit in layer $\ell$ of the neural network, given an input vector $x=(x_1, \ldots, x_{M_0})$.

We build the DRN as follows. First we define its architecture. The DRN consists of $L$ layers. Input and hidden layers ($0 \leq \ell \leq L-1$) each contain $N_\ell := 2 M_{\ell} + 2$ units, including $M_{\ell}$ excitatory units, $M_{\ell}$ inhibitory units and two bias units. The output layer contains $N_L := M_L$ linear units. To fully determine the DRN, it remains to choose the values of the conductances ($g_{jk}^{(\ell)}$), as well as the amplification factors ($A^{(\ell)}$) for input and hidden layers. The key idea of our construction is that for each hidden unit $s_k^{(\ell)}$ of the ReLU neural network, we associate two DRN units $v_{2k}^{(\ell)}$ and $v_{2k+1}^{(\ell)}$, choosing the conductances and amplification factors in such a way that $v_{2k}^{(\ell)} = - v_{2k+1}^{(\ell)} \approx A^{(\ell)} s_k^{(\ell)}$.

To do that, we denote the biases of the neural network for convenience by $w_{0k}^{(\ell)} := b_k^{(\ell)}$. We also introduce a dimensionless number $\gamma > 0$ ; later we will let $\gamma \to 0$ to establish the equivalence of the DRN and the ReLU neural network. For every $\ell$ such that $1 \leq \ell \leq L-1$, given pairs of units $(v_{2j}^{(\ell-1)}, v_{2j+1}^{(\ell-1)})$ and $(v_{2k}^{(\ell)}, v_{2k+1}^{(\ell)})$, we define the conductances of the resistors interconnecting them as follows:
\begin{align}
g_{2j,2k}^{(\ell)} & := \max \left( 0,\gamma^\ell w_{jk}^{(\ell)} \right), \\
g_{2j,2k+1}^{(\ell)} & := \max \left( 0,-\gamma^\ell w_{jk}^{(\ell)} \right), \\
g_{2j+1,2k}^{(\ell)} & := \max \left( 0,-\gamma^\ell w_{jk}^{(\ell)} \right), \\
g_{2j+1,2k+1}^{(\ell)} & := \max \left( 0,\gamma^\ell w_{jk}^{(\ell)} \right).
\end{align}
Similarly for the output layer, for every $(j,k)$ such that $0 \leq j \leq M_{L-1}-1$ and $0 \leq k \leq M_L-1$, we define 
\begin{align}
g_{2j,k}^{(L)} & := \max \left( 0,\gamma^L w_{jk}^{(L)} \right), \\
g_{2j+1,k}^{(L)} & := \max \left( 0,-\gamma^L w_{jk}^{(L)} \right).
\end{align}
The conductances thus defined are all non-negative. The purpose of the parameter $\gamma \ll 1$ is to ensure that, at each hidden layer $\ell$, the signals coming from the next layer ($\ell+1$), which scale approximately as $\gamma^{\ell+1}$, are negligible compared to the signals coming from the previous layer ($\ell-1$), which scale approximately as $\gamma^{\ell}$. This guarantees that the network effectively operates in a feedforward manner.

It remains to define the layerwise amplification factors and the unitwise leak conductances. For every $\ell$ such that $1 \leq \ell \leq L$, we define
\begin{equation}
\label{eq:layer-gain}
a^{(\ell)} := \max_{1 \leq k \leq M_{\ell}} \left( \sum_{j=0}^{M_{\ell-1}} |w_{jk}^{(\ell)}| \right),
\end{equation}
and for every $(\ell,k)$ such that $1 \leq \ell \leq L$ and $1 \leq k \leq M_\ell$, we define the `leak conductances' of units $v_{2k}^{(\ell)}$ and $v_{2k+1}^{(\ell)}$ as
\begin{equation}
g_{2k}^{(\ell)} := g_{2k+1}^{(\ell)} := \gamma^\ell \left( a^{(\ell)} - \sum_{j=0}^{M_{\ell-1}} |w_{jk}^{(\ell)}| \right).
\end{equation}
The conductances $g_{2k}^{(\ell)}$ and $g_{2k+1}^{(\ell)}$ are also non-negative, by definition of $a^{(\ell)}$. Finally, we define the layerwise amplification factors (for input and hidden layers) as
\begin{equation}
A^{(\ell)} := a^{(\ell+1)} \times a^{(\ell+2)} \times \cdots \times a^{(L)}, \qquad 0 \leq \ell \leq L-1.
\end{equation}

\begin{restatable}[Approximation of a ReLU neural network with a DRN]{thm}{approximationrelunndrn}
\label{thm:drn-approximate-relu-nn}
Suppose that, given an input vector $x=(x_1,x_2,\ldots,x_{M_0})$, we set the input voltage sources to
\begin{equation}
\label{eq:thm-initial-condition}
v_{2k}^{(0)} := A^{(0)} x_k, \qquad v_{2k+1}^{(0)} := - A^{(0)} x_k, \qquad 1 \leq k \leq M_0.
\end{equation}
Denote $\|x\|_{\infty} := \max_{1 \leq k \leq M_0} |x_k|$. Then, in the limit $\gamma \to 0$, for every $(\ell,k)$ such that $1 \leq \ell \leq L-1$ and $1 \leq k \leq M_\ell$, we have
\begin{gather}
\label{eq:thm-approx}
v_{2k}^{(\ell)} = A^{(\ell)} s_k^{(\ell)} + O(\gamma \|x\|_{\infty}), \\
v_{2k+1}^{(\ell)} = - A^{(\ell)} s_k^{(\ell)} + O(\gamma \|x\|_{\infty}).
\end{gather}
For the output layer, we have
\begin{equation}
v_{k}^{(L)} = s_k^{(L)} + O(\gamma \|x\|_{\infty}), \quad 1 \leq k \leq M_L.
\end{equation}
\end{restatable}

See Appendix~\ref{sec:proof:drn-approximate-relu-nn} for a proof. Namely, when $\gamma \to 0$, the states of the excitatory units of the DRN tend to the states of the units of the ReLU neural network rescaled by $A^{(\ell)}$. In particular, the states of output units of the DRN are equal to those of the neural network, up to $O(\gamma \|x\|_{\infty})$. In other words, the function implemented by the DRN approximates the function implemented by the ReLU neural network.

Using a symmetry argument, one can prove that the DRN thus built has the property that $v_{2k+1}^{(\ell)} = - v_{2k}^{(\ell)}$ for every $(\ell,k)$. Although it might seem wasteful and redundant to have two units encode the same piece of information, we recall that doubling the number of units per layer (using excitatory and inhibitory units) is also what allows us to overcome the constraint of non-negative conductances in resistive networks.

Technically, the reason why the DRN behaves essentially like a feedforward (ReLU) network is that we scale the conductances of each layer $\ell$ by $\gamma^\ell$, using $\gamma \ll 1$. While this trick is mathematically convenient to derive Theorem~\ref{thm:drn-approximate-relu-nn}, it would also likely be impractical for deep networks as the range of conductance values will span multiple orders of magnitude, proportional to the number of layers $L$. It is noteworthy, however, that neural networks with a single hidden layer ($L=2$) are also universal approximators (Theorem~\ref{thm:relu-nn-universal}). The DRNs corresponding to such neural networks have a single hidden layer, therefore the range spanned by their conductance values is smaller.

\section{Simulations}
\label{sec:simulations}

In this section, we illustrate the expressivity of nonlinear resistive networks in simulations. We show that DRNs can be trained to achieve comparable performance to ReLU neural networks on the MNIST, Kuzushiji-MNIST and Fashion-MNIST datasets. We train DRNs with one, two and three hidden layers, denoted by DRN-1H, DRN-2H and DRN-3H, respectively. We compare their performances to their corresponding ReLU neural networks. Each DRN has 1568 input units ($2\times28\times28$), 10 output units corresponding to the ten classes of each of the three datasets, and 1024 units per hidden layer. Each ReLU neural network has 784 input units ($28\times28$), 10 output units and 512 units per hidden layer. We use the algorithm of \cite{scellier2024fast} to compute the DRN steady states necessary to extract the weight (conductance) gradients to perform stochastic gradient descent. A full description of the models, algorithms, and hyperparameters used for training is provided in Appendix~\ref{sec:simulation-details}. Table~\ref{table:results} reports the results. Our code to reproduce the results is available at \url{https://github.com/rain-neuromorphics/energy-based-learning}.

\begin{table}[h!]
\caption{Comparison of deep resistive networks (DRNs) and their equivalent-size ReLU neural networks (ReLU NNs). We train DRNs and ReLU NNs with one, two and three hidden layers on MNIST, F-MNIST and K-MNIST. For each experiment, we perform five runs and we report the mean values and std values of the test error rates (in \%). See Appendix~\ref{sec:simulation-details} for the architectural details of the DRN and ReLU NN models, as well as the algorithms used to train them.}
\label{table:results}
\vskip 0.15in
\begin{center}
\begin{small}
\begin{sc}
\begin{tabular}{ccc}
    \toprule
    Dataset & Network Model & Test Error (\%) \\
    \midrule
    & DRN-1H & 1.51 $\pm$ 0.04 \\
    & ReLU NN-1H & 1.67 $\pm$ 0.03 \\
    \cmidrule(r){2-3}
    \multirow{2}{*}{MNIST} & DRN-2H & 1.46 $\pm$ 0.05 \\
    & ReLU NN-2H & 1.34 $\pm$ 0.01 \\
    \cmidrule(r){2-3}
    & DRN-3H & 1.47 $\pm$ 0.08 \\
    & ReLU NN-3H & 1.48 $\pm$ 0.03 \\
    \midrule
    & DRN-1H & 7.67 $\pm$ 0.13 \\
    & ReLU NN-1H & 8.83 $\pm$ 0.11 \\
    \cmidrule(r){2-3}
    Kuzushiji & DRN-2H & 8.27 $\pm$ 0.02 \\
    MNIST & ReLU NN-2H & 8.14 $\pm$ 0.13 \\
    \cmidrule(r){2-3}
    & DRN-3H & 8.39 $\pm$ 0.24 \\
    & ReLU NN-3H & 7.99 $\pm$ 0.23 \\
    \midrule
    & DRN-1H & 10.31 $\pm$ 0.05 \\
    & ReLU NN-1H & 10.20 $\pm$ 0.11 \\
    \cmidrule(r){2-3}
    Fashion & DRN-2H & 9.88 $\pm$ 0.11 \\
    MNIST & ReLU NN-2H & 9.42 $\pm$ 0.22 \\
    \cmidrule(r){2-3}
    & DRN-3H & 9.79 $\pm$ 0.25 \\
    & ReLU NN-3H & 9.55 $\pm$ 0.14 \\
    \bottomrule
\end{tabular}
\end{sc}
\end{small}
\end{center}
\vskip -0.1in
\end{table} 

As discussed earlier, one limiting assumption in our mathematical derivation of Theorem~\ref{thm:drn-approximate-relu-nn} is the requirement to scale the weights (conductances) of layer $\ell$ by a factor $\gamma^\ell$, where $\gamma \ll 1$. Interestingly, however, our simulation results suggest that this scaling is not necessary in practice. In Table~\ref{table:conductance-span}, we report for each DRN model and for each layer $\ell$ the mean and standard deviation of the conductance matrix $g^{(\ell)}$ after training. As one can see, the difference in conductance values across layers are not significant, demonstrating that conductance matrices do not actually need to span over multiple orders of magnitude.

\begin{table}[ht!]
\caption{Conductance values of the DRN models after training (DRN-1H, DRN-2H, and DRN-3H). We report the mean, standard deviation (std), and maximum values of the conductance matrices ($g^{(k)}$) for each of the three DRN models.}
\label{table:conductance-span}
\vskip 0.15in
\begin{center}
\begin{small}
\begin{sc}
\begin{tabular}{cccccc}
\toprule
& & $g^{(1)}$ & $g^{(2)}$ & $g^{(3)}$ & $g^{(4)}$ \\
\midrule
\multirow{3}{*}{DRN-1H} & mean & 0.0077 & 0.0091 & & \\
& std & 0.0102 & 0.0327 & & \\
& max & 0.2069 & 0.3295 & & \\
\midrule
\multirow{3}{*}{DRN-2H} & mean & 0.0071 & 0.0074 & 0.0024 & \\
& std & 0.0082 & 0.0102 & 0.0249 & \\
& max & 0.0996 & 0.1459 & 0.4897 & \\
\midrule
\multirow{3}{*}{DRN-3H} & mean & 0.0076 & 0.0079 & 0.0053 & 0.0034 \\
& std & 0.0086 & 0.0113 & 0.0124 & 0.0230 \\
& max & 0.1276 & 0.4104 & 0.8345 & 0.4125 \\
\bottomrule
\end{tabular}
\end{sc}
\end{small}
\end{center}
\vskip -0.1in
\end{table}

\section{Discussion}

Since the seminal work of Cybenko \citep{cybenko1989approximation}, the concept of universal function approximation has been extensively studied in the context of neural networks. On the other hand, the computational capabilities of analog physical systems that are not isomorphic to neural networks remain largely unexplored, despite recent efforts, for example, in optical systems \citep{marcucci2020theory}. In this study, we have examined the expressive power of resistor networks, which use voltage sources as inputs and variable resistors as trainable weights. These networks can be trained by gradient descent using local learning rules \citep{kendall2020training,anisetti2024frequency}, and prototypical self-learning resistor networks based on such principles have been built \citep{dillavou2022demonstration,dillavou2024machine}. Our study advances our understanding of the computational properties of these resistor networks, and further highlights their potential. We have shown that, under suitable assumptions, resistor networks equipped with diodes and voltage-controlled voltage sources -- referred to as nonlinear resistive networks -- can approximate any continuous function. Central to our proof is the demonstration that a subclass of these electrical networks -- so-called `deep resistive networks' -- can approximate arbitrary ReLU neural networks with any desired precision. Our result thus offers a method for translating a ReLU neural network into an approximately equivalent DRN.

Nonetheless, our theoretical results rely on several assumptions, some of which are impractical. First, to derive our results, we have assumed ideal diodes, but real-world diodes deviate significantly from our idealized model. Future works could attempt to adapt our approach to networks with more realistic current-voltage ($i$-$v$) characteristics, for example, by adapting the steps from Eq.~\eqref{eq:master-equation} to Eq.~\eqref{eq:relu-nonlinear-activation-function} in our derivation (Appendix~\ref{sec:proofs}). Second, our construction uses VCVSs with very large gains ($A^{(0)} \gg 1$) to amplify the DRN's input signals, which would be impractical in an actual circuit. A solution proposed by Kendall \emph{et al.} \citep{kendall2020training} is to use layerwise `bidirectional amplifiers' with smaller gains, as a replacement for the input VCVSs (see Appendix~\ref{sec:bidirectional-amplifiers} for details). Third, we have assumed that the conductances of the variable resistors can be adjusted to arbitrary non-negative values - arbitrarily large or arbitrarily small -- but in practice, real-world memristors have quantized conductance values and operate within a finite range. Fourth, our model is deterministic, whereas real-world analog electrical networks are inherently noisy. Despite these assumptions, our universal approximation theorem represents the first theoretical result of its kind for analog resistor networks trainable via EP, and it could serve as a basis for deriving more practical theorems requiring fewer or more realistic assumptions.

We also note that nonlinear resistive networks are closely related to continuous Hopfield networks (CHNs) \citep{hopfield1984neurons} (for a comparison, see \citep[Appendix E]{scellier2024fast}). In particular, the idea of rescaling the weights layerwise, so the DRN behaves approximately as a ReLU neural network, is reminiscent of the work by Xie and Seung \citep{xie2003equivalence} where it is shown that layered neural networks can be approximated by layered CHNs using a similar technique. The similarity between nonlinear resistive networks and CHNs is appealing, particularly as recent studies have shown through simulations that CHNs can be effectively trained using equilibrium propagation on datasets like CIFAR10, CIFAR100, and ImageNet $32 \times 32$, yielding promising results \citep{laborieux2021scaling,laborieux2022holomorphic,scellier2023energy}.

\section*{Acknowledgements}

The authors thank Jack Kendall, Maxence Ernoult, Mohammed Fouda, Suhas Kumar, Vidyesh Anisetti, Andrea Liu, Axel Laborieux and Tim De Ryck for discussions.

\bibliographystyle{icml2024}
\bibliography{biblio}

\clearpage
\appendix

\onecolumn  

\newpage
\section{Proofs}
\label{sec:proofs}

In this appendix, we prove the theoretical results, apart from Theorem~\ref{thm:relu-nn-universal} which is well known -- we refer to \cite{leshno1993multilayer} for a proof of it. We proceed as follows. First, we prove Lemma~\ref{lma:deep-resistive-network}. Next, we prove Lemma~\ref{lma:maximum-principle} using Lemma~\ref{lma:deep-resistive-network}. Then, we prove Theorem~\ref{thm:drn-approximate-relu-nn} using Lemma~\ref{lma:deep-resistive-network} and Lemma~\ref{lma:maximum-principle}. Finally, we prove Theorem~\ref{thm:nrn-universal} using Theorem~\ref{thm:relu-nn-universal} and Theorem~\ref{thm:drn-approximate-relu-nn}.

\subsection{Proof of Lemma~\ref{lma:deep-resistive-network}}
\label{sec:proof:deep-resistive-network}

For clarity, we repeat the lemma.

\equationsdrn*

\begin{proof}[Proof of Lemma~\ref{lma:deep-resistive-network}]
Consider unit $k$ in layer $\ell$, with $0 < \ell \leq L$ and $0 \leq k \leq N_\ell-1$. First we consider the case $\ell < L$ of a hidden layer (we will consider the case $\ell=L$ of the output layer later). We denote by $i_k^{(\ell)}$ the current flowing through the diode to the unit's node. The current flowing from ground to the unit's node through the resistor is $- g_k^{(\ell)} v_k^{(\ell)}$. We also denote by $i_{jk}^{(\ell)}$ the current flowing from the $j$-th unit of layer $\ell-1$ to the $k$-th unit of layer $\ell$. Kirchhoff's current law applied to the $k$-th unit of layer $\ell$ tells us that
\begin{equation}
i_k^{(\ell)} - g_k^{(\ell)} v_k^{(\ell)} + \sum_{j=0}^{N_{\ell-1}-1} i_{jk}^{(\ell)} - \sum_{j=2}^{N_{\ell+1}-1} i_{kj}^{(\ell+1)} = 0.
\end{equation}
Furthermore, by Ohm's law we have $i_{jk}^{(\ell)} = g_{jk}^{(\ell)} \left( v_j^{(\ell-1)}-v_k^{(\ell)} \right)$ and $i_{kj}^{(\ell+1)} = g_{kj}^{(\ell+1)} \left( v_k^{(\ell)}-v_j^{(\ell+1)} \right)$. Therefore
\begin{equation}
i_k^{(\ell)} - g_k^{(\ell)} v_k^{(\ell)} + \sum_{j=0}^{N_{\ell-1}-1} g_{jk}^{(\ell)} \left( v_j^{(\ell-1)}-v_k^{(\ell)} \right)
- \sum_{j=2}^{N_{\ell+1}-1} g_{kj}^{(\ell+1)} \left( v_k^{(\ell)}-v_j^{(\ell+1)} \right) = 0.
\label{eq:kirchhoff}
\end{equation}
Solving \eqref{eq:kirchhoff} for $v_k^{(\ell)}$ yields
\begin{equation}
\label{eq:master-equation}
v_k^{(\ell)} = 
\frac{i_k^{(\ell)} + \sum_{j=0}^{N_{\ell-1}-1} g_{jk}^{(\ell)} v_j^{(\ell-1)} + \sum_{j=2}^{N_{\ell+1}-1} g_{kj}^{(\ell+1)} v_j^{(\ell+1)}}{g_k^{(\ell)} + \sum_{j=0}^{N_{\ell-1}-1} g_{jk}^{(\ell)} + \sum_{j=2}^{N_{\ell+1}-1} g_{kj}^{(\ell+1)}}
\end{equation}
Now we use the characteristics of the diode between the unit's node and ground. First we consider the case of an excitatory unit ($v_k^{(\ell)} \geq 0$). We distinguish between two subcases: either $v_k^{(\ell)} > 0$ and $i_k^{(\ell)} = 0$ (the diode is in the `off state'), or $v_k^{(\ell)} = 0$ and $i_k^{(\ell)} \geq 0$ (the diode is in the `on state'). In the first subcase (off state) we have, using \eqref{eq:master-equation},
\begin{equation}
\label{eq:first-case}
0 < v_k^{(\ell)} = \frac{\sum_{j=0}^{N_{\ell-1}-1} g_{jk}^{(\ell)} v_j^{(\ell-1)} + \sum_{j=2}^{N_{\ell+1}-1} g_{kj}^{(\ell+1)} v_j^{(\ell+1)}}{g_k^{(\ell)} + \sum_{j=0}^{N_{\ell-1}-1} g_{jk}^{(\ell)} + \sum_{j=2}^{N_{\ell+1}-1} g_{kj}^{(\ell+1)}}.
\end{equation}
In the second subcase (on state) we have, using \eqref{eq:master-equation} again,
\begin{equation}
\label{eq:second-case}
0 = v_k^{(\ell)} \geq \frac{\sum_{j=0}^{N_{\ell-1}-1} g_{jk}^{(\ell)} v_j^{(\ell-1)} + \sum_{j=2}^{N_{\ell+1}-1} g_{kj}^{(\ell+1)} v_j^{(\ell+1)}}{g_k^{(\ell)} + \sum_{j=0}^{N_{\ell-1}-1} g_{jk}^{(\ell)} + \sum_{j=2}^{N_{\ell+1}-1} g_{kj}^{(\ell+1)}}.
\end{equation}
From \eqref{eq:first-case} and \eqref{eq:second-case} it follows that both subcases are captured by a single formula:
\begin{equation}
\label{eq:relu-nonlinear-activation-function}
v_k^{(\ell)} =
\max \left( 0, \frac{\sum_{j=0}^{N_{\ell-1}-1} g_{jk}^{(\ell)} v_j^{(\ell-1)} + \sum_{j=2}^{N_{\ell+1}-1} g_{kj}^{(\ell+1)} v_j^{(\ell+1)}}{g_k^{(\ell)} + \sum_{j=0}^{N_{\ell-1}-1} g_{jk}^{(\ell)} + \sum_{j=2}^{N_{\ell+1}-1} g_{kj}^{(\ell+1)}} \right).
\end{equation}
Similarly, an inhibitory unit satisfies
\begin{equation}
v_k^{(\ell)} =
\min \left( 0, \frac{\sum_{j=0}^{N_{\ell-1}-1} g_{jk}^{(\ell)} v_j^{(\ell-1)} + \sum_{j=2}^{N_{\ell+1}-1} g_{kj}^{(\ell+1)} v_j^{(\ell+1)}}{g_k^{(\ell)} + \sum_{j=0}^{N_{\ell-1}-1} g_{jk}^{(\ell)} + \sum_{j=2}^{N_{\ell+1}-1} g_{kj}^{(\ell+1)}} \right).
\end{equation}
The case of output nodes ($\ell = L$) is identical, with $N_{L+1} = 0$ (meaning the layer of index $L+1$ is empty) and $i_k^{(L)} = 0$ (no diode). Equation~\eqref{eq:master-equation} in this case is rewritten as
\begin{equation}
v_k^{(L)} = 
\frac{\sum_{j=0}^{N_{L-1}-1} g_{jk}^{(L)} v_j^{(L-1)}}{g_k^{(L)} + \sum_{j=0}^{N_{L-1}-1} g_{jk}^{(L)}}
\end{equation}
\end{proof}

\subsection{Proof of Lemma \ref{lma:maximum-principle}}
\label{sec:proof:lma:maximum-principle}

\maximumprinciple*

\begin{proof}[Proof of Lemma \ref{lma:maximum-principle}]
To prove \eqref{eq:to-prove}, we proceed by induction on the layer index $\ell$, starting from the output layer ($\ell=L$), down to the first hidden layer ($\ell=1$). We start with $\ell=L$. Using Lemma~\ref{lma:deep-resistive-network} we have for every output unit $k$,
\begin{align}
|v_k^{(L)}| & \leq \left| \frac{\sum_{j=0}^{N_{L-1}-1} g_{jk}^{(L)} v_j^{(L-1)}}{g_k^{(L)} + \sum_{j=0}^{N_{L-1}-1} g_{jk}^{(L)}} \right| \\
& \leq \frac{\sum_{j=0}^{N_{L-1}-1} g_{jk}^{(L)} |v_j^{(L-1)}|}{g_k^{(L)} + \sum_{j=0}^{N_{L-1}-1} g_{jk}^{(L)}} \\
& \leq \frac{\sum_{j=0}^{N_{L-1}-1} g_{jk}^{(L)} |v_j^{(L-1)}|}{\sum_{j=1}^{N_{L-1}-1} g_{jk}^{(L)}} \\
& \leq \frac{\sum_{j=0}^1 g_{jk}^{(L)} A^{(L-1)} + \sum_{j=2}^{N_{L-1}-1} g_{jk}^{(L)} v_{\rm max}^{(L-1)}}{\sum_{j=0}^{N_{L-1}-1} g_{jk}^{(L)}} \\
& \leq \max \left( v_{\rm max}^{(L-1)}, A^{(L-1)} \right).
\end{align}
Since this holds for every output unit ($1 \leq k \leq N_L$), it follows that
\begin{equation}
v_{\rm max}^{(L)} \leq \max \left( v_{\rm max}^{(L-1)}, A^{(L-1)} \right).
\end{equation}
This completes the initialization step of the proof by induction.

Suppose now that \eqref{eq:to-prove} holds for layer $\ell+1$ where $2 \leq \ell+1 \leq L$, that is,
\begin{equation}
v_{\rm max}^{(\ell+1)} \leq \max \left( v_{\rm max}^{(\ell)}, A^{(\ell)}, \ldots, A^{(L-1)} \right),
\end{equation}
and let us show that the property holds for layer $\ell$. To show this, we distinguish between two cases: either $v_{\rm max}^{(\ell)} \leq \max \left( A^{(\ell)}, \ldots, A^{(L-1)} \right)$ or $v_{\rm max}^{(\ell)} > \max \left( A^{(\ell)}, \ldots, A^{(L-1)} \right)$. The first case directly leads to $v_{\rm max}^{(\ell)} \leq \max \left( v_{\rm max}^{(\ell-1)}, A^{(\ell-1)}, \ldots, A^{(L-1)} \right)$, as desired. The second case implies that
\begin{equation}
\label{eq:proof-3}
v_{\rm max}^{(\ell+1)} \leq v_{\rm max}^{(\ell)}.
\end{equation}
Next, using Lemma~\ref{lma:deep-resistive-network} again,
\begin{align}
|v_k^{(\ell)}| & \leq \left| \frac{\sum_{j=0}^{N_{\ell-1}-1} g_{jk}^{(\ell)} v_j^{(\ell-1)} + \sum_{j=2}^{N_{\ell+1}-1} g_{kj}^{(\ell+1)} v_j^{(\ell+1)}}{g_k^{(\ell)} + \sum_{j=0}^{N_{\ell-1}-1} g_{jk}^{(\ell)} + \sum_{j=2}^{N_{\ell+1}-1} g_{kj}^{(\ell+1)}} \right| \\
& \leq \frac{\sum_{j=0}^{N_{\ell-1}-1} g_{jk}^{(\ell)} |v_j^{(\ell-1)}| + \sum_{j=2}^{N_{\ell+1}-1} g_{kj}^{(\ell+1)} |v_j^{(\ell+1)}|}{g_k^{(\ell)} + \sum_{j=0}^{N_{\ell-1}-1} g_{jk}^{(\ell)} + \sum_{j=2}^{N_{\ell+1}-1} g_{kj}^{(\ell+1)}} \\
& \leq \frac{\sum_{j=0}^{N_{\ell-1}-1} g_{jk}^{(\ell)} |v_j^{(\ell-1)}| + \sum_{j=2}^{N_{\ell+1}-1} g_{kj}^{(\ell+1)} |v_j^{(\ell+1)}|}{\sum_{j=0}^{N_{\ell-1}-1} g_{jk}^{(\ell)} + \sum_{j=2}^{N_{\ell+1}-1} g_{kj}^{(\ell+1)}} \\
& \leq \frac{\sum_{j=0}^1 g_{jk}^{(\ell)} A^{(\ell-1)} + \sum_{j=2}^{N_{\ell-1}-1} g_{jk}^{(\ell)} v_{\rm max}^{(\ell-1)} + \sum_{j=2}^{N_{\ell+1}-1} g_{kj}^{(\ell+1)} v_{\rm max}^{(\ell+1)}}{\sum_{j=0}^{N_{\ell-1}-1} g_{jk}^{(\ell)} + \sum_{j=2}^{N_{\ell+1}-1} g_{kj}^{(\ell+1)}} \\
& \leq \frac{\sum_{j=0}^1 g_{jk}^{(\ell)} A^{(\ell-1)} + \sum_{j=2}^{N_{\ell-1}-1} g_{jk}^{(\ell)} v_{\rm max}^{(\ell-1)} + \sum_{j=2}^{N_{\ell+1}-1} g_{kj}^{(\ell+1)} v_{\rm max}^{(\ell)}}{\sum_{j=0}^{N_{\ell-1}-1} g_{jk}^{(\ell)} + \sum_{j=2}^{N_{\ell+1}-1} g_{kj}^{(\ell+1)}},
\end{align}
where we have used Eq.~\eqref{eq:proof-3}. Next we select the index $k$ ($2 \leq k \leq N_\ell-1$) that achieves the largest magnitude, that is, $|v_k^{(\ell)}| = v_{\rm max}^{(\ell)}$. We get
\begin{equation}
v_{\rm max}^{(\ell)} \leq \frac{\sum_{j=0}^1 g_{jk}^{(\ell)} A^{(\ell-1)} + \sum_{j=2}^{N_{\ell-1}-1} g_{jk}^{(\ell)} v_{\rm max}^{(\ell-1)} + \sum_{j=2}^{N_{\ell+1}-1} g_{kj}^{(\ell+1)} v_{\rm max}^{(\ell)}}{\sum_{j=0}^{N_{\ell-1}-1} g_{jk}^{(\ell)} + \sum_{j=2}^{N_{\ell+1}-1} g_{kj}^{(\ell+1)}}.
\end{equation}
Rearranging the terms, we successively obtain
\begin{gather}
\left( \sum_{j=0}^{N_{\ell-1}-1} g_{jk}^{(\ell)} + \sum_{j=2}^{N_{\ell+1}-1} g_{kj}^{(\ell+1)} \right) v_{\rm max}^{(\ell)}
\leq \sum_{j=0}^1 g_{jk}^{(\ell)} A^{(\ell-1)} + \sum_{j=2}^{N_{\ell-1}-1} g_{jk}^{(\ell)} v_{\rm max}^{(\ell-1)} + \sum_{j=2}^{N_{\ell+1}-1} g_{kj}^{(\ell+1)} v_{\rm max}^{(\ell)}, \\
\left( \sum_{j=0}^{N_{\ell-1}-1} g_{jk}^{(\ell)} \right) v_{\rm max}^{(\ell)} \leq \sum_{j=0}^1 g_{jk}^{(\ell)} A^{(\ell-1)} + \sum_{j=2}^{N_{\ell-1}-1} g_{jk}^{(\ell)} v_{\rm max}^{(\ell-1)}, \\
v_{\rm max}^{(\ell)} \leq \frac{\sum_{j=0}^1 g_{jk}^{(\ell)} A^{(\ell-1)} + \sum_{j=2}^{N_{\ell-1}-1} g_{jk}^{(\ell)} v_{\rm max}^{(\ell-1)}}{\sum_{j=0}^{N_{\ell-1}-1} g_{jk}^{(\ell)}} \leq \max \left( A^{(\ell-1)}, v_{\rm max}^{(\ell-1)} \right),
\end{gather}
where we have used that a weighted mean of terms with positive weights is bounded above by the max of the terms. Again this leads to
\begin{equation}
v_{\rm max}^{(\ell)} \leq \max \left( v_{\rm max}^{(\ell-1)}, A^{(\ell-1)}, \ldots, A^{(L-1)} \right),
\end{equation}
as desired. This completes the induction step, and therefore completes the proof.
\end{proof}

\subsection{Proof of Theorem~\ref{thm:drn-approximate-relu-nn}}
\label{sec:proof:drn-approximate-relu-nn}

\approximationrelunndrn*

\begin{proof}[Proof of Theorem~\ref{thm:drn-approximate-relu-nn}]
We prove property \eqref{eq:thm-approx} by induction on $\ell$. It is true for $\ell=0$ thanks to \eqref{eq:thm-initial-condition}. Suppose that \eqref{eq:thm-approx} is true for some $\ell-1 \geq 0$ and let us prove it at the rank $\ell$. Let $k$ such that $1 \leq k \leq M_\ell$, and consider the excitatory unit $v_{2k}^{(\ell)}$ (the case of inhibitory unit $v_{2k+1}^{(\ell)}$, and the case of linear output unit $v_{k}^{(L)}$ if $\ell=L$, are similar). Recall from Lemma~\ref{lma:deep-resistive-network} that
\begin{equation}
\label{eq:big-expression}
v_{2k}^{(\ell)} =
\max \left( 0, \frac{\sum_{j=0}^{N_{\ell-1}-1} g_{j,2k}^{(\ell)} v_j^{(\ell-1)} + \sum_{j=2}^{N_{\ell+1}-1} g_{2k,j}^{(\ell+1)} v_j^{(\ell+1)}}{g_{2k}^{(\ell)} + \sum_{j=0}^{N_{\ell-1}-1} g^{(\ell)}_{j,2k} + \sum_{j=2}^{N_{\ell+1}-1} g^{(\ell+1)}_{2k,j}} \right).
\end{equation}
First we calculate the denominator of expression \eqref{eq:big-expression}. By distinguishing between the two cases $w^{(\ell)}_{jk} \geq 0$ and $w^{(\ell)}_{jk} \leq 0$, it is easy to see that we have in both cases
\begin{equation}
g^{(\ell)}_{2j,2k} + g^{(\ell)}_{2j+1,2k} = \gamma^\ell |w^{(\ell)}_{jk}|.
\end{equation}
Summing from $j=0$ to $j=M_{\ell-1}$, we get
\begin{equation}
\sum_{j=0}^{N_{\ell-1}-1} g^{(\ell)}_{j,2k} = \gamma^\ell \sum_{j=0}^{M_{\ell-1}} |w_{jk}^{(\ell)}|.
\end{equation}
Recalling that, by definition, $g_{2k}^{(\ell)} := \gamma^\ell \left( a^{(\ell)} - \sum_{j=0}^{M_{\ell-1}} |w_{jk}^{(\ell)}| \right)$, it follows that
\begin{equation}
g_{2k}^{(\ell)} + \sum_{j=0}^{N_{\ell-1}-1} g^{(\ell)}_{j,2k} =\gamma^\ell a^{(\ell)}.
\end{equation}
Thus, we get the following formula for the denominator in \eqref{eq:big-expression}:
\begin{equation}
\label{eq:proof-1}
\underbrace{g_{2k}^{(\ell)} + \sum_{j=0}^{N_{\ell-1}-1} g^{(\ell)}_{j,2k}}_{=\gamma^\ell a^{(\ell)}} + \underbrace{\sum_{j=2}^{N_{\ell+1}-1} g^{(\ell+1)}_{2k,j}}_{=O(\gamma^{\ell+1})} = \gamma^\ell a^{(\ell)} + O(\gamma^{\ell+1}).
\end{equation}
Next, we calculate the numerator of \eqref{eq:big-expression}. Similarly, whether $w^{(\ell)}_{jk} \geq 0$ or $w^{(\ell)}_{jk} \leq 0$, we have in both cases, using the inductive hypothesis,
\begin{equation}
g^{(\ell)}_{2j,2k} v_{2j}^{(\ell-1)} + g^{(\ell)}_{2j+1,2k} v_{2j+1}^{(\ell-1)} = \gamma^\ell A^{(\ell-1)} w^{(\ell)}_{jk} s_j^{(\ell-1)} + O(\gamma^{\ell+1} \|x\|_{\infty}).
\end{equation}
This identity also holds for $j=0$ by defining $s_0^{(\ell-1)} := 1$ to include the case of the biases. Summing from $j=0$ to $j=M_{\ell-1}$, we get
\begin{equation}
\sum_{j=0}^{N_{\ell-1}-1} g_{j,2k}^{(\ell)} v_{j}^{(\ell-1)} + \underbrace{\sum_{j=2}^{N_{\ell+1}-1} g_{2k,j}^{(\ell+1)} v_{j}^{(\ell+1)}}_{=O(\gamma^{\ell+1} \|x\|_{\infty}) \; \text{by Lemma~\ref{lma:maximum-principle}}}
= \gamma^\ell A^{(\ell-1)} \left( \sum_{j=0}^{M_{\ell-1}} w_{jk}^{(\ell)} s_j^{(\ell-1)} \right) + O(\gamma^{\ell+1} \|x\|_{\infty}).
\label{eq:proof-2}
\end{equation}
Here we have used the maximum principle (Lemma~\ref{lma:maximum-principle}) to derive the bound $O(\gamma^{\ell+1} \|x\|_{\infty})$ for the $v_{j}^{(\ell+1)}$ term. Finally, substituting \eqref{eq:proof-1} and \eqref{eq:proof-2} into \eqref{eq:big-expression}, we obtain the following equation for $v_{2k}^{(\ell)}$:
\begin{equation}
v_{2k}^{(\ell)} = \max \left( 0, \frac{A^{(\ell-1)}}{a^{(\ell)}} \left( \sum_{j=1}^{M_{\ell-1}} w_{jk}^{(\ell)} s_j^{(\ell-1)} + b_k^{(\ell)} \right) \right) + O(\gamma \|x\|_{\infty})
= A^{(\ell)} s_k^{(\ell)} + O(\gamma \|x\|_{\infty}),
\end{equation}
where we have used that $A^{(\ell-1)} = a^{(\ell)} A^{(\ell)}$, by definition. The case of $v_{2k+1}^{(\ell)}$ (inhibitory unit) is similar: we obtain
\begin{equation}
v_{2k+1}^{(\ell)} = - A^{(\ell)} s_k^{(\ell)} + O(\gamma \|x\|_{\infty}).
\end{equation}
Hence, property \eqref{eq:thm-approx} is true at the rank $\ell$. (The specific case of the linear output layer can be treated similarly.) This completes the proof by induction.
\end{proof}

\subsection{Proof of Theorem~\ref{thm:nrn-universal}}

\nonlinearresistivenetwork*

\begin{proof}[Proof of Theorem~\ref{thm:nrn-universal}]
Let $f: \mathbb{R}^q \to \mathbb{R}^r$ be a continuous function. Let $C$ be a compact subset of $\mathbb{R}^q$, and let $\epsilon>0$. By Theorem~\ref{thm:relu-nn-universal}, there exists a ReLU neural network $\mathcal{N}$ with $q$ inputs and $r$ outputs such that the function $G_{\mathcal{N}}$ that it implements satisfies
\begin{equation}
\| G_{\mathcal{N}}(x) - f(x) \| \leq \frac{\epsilon}{2}, \qquad \forall x \in C.
\end{equation}
Let $\mathcal{R}_\gamma$ be the DRN approximator of $\mathcal{N}$ with parameter $\gamma$, as defined in Section~\ref{sec:universal-approximators}. By Theorem~\ref{thm:drn-approximate-relu-nn}, the function $F_{\mathcal{R}_\gamma}$ implemented by $\mathcal{R}_\gamma$ satisfies
\begin{equation}
\| F_{\mathcal{R}_\gamma}(x) - G_{\mathcal{N}}(x) \| = O \left( \gamma \|x\|_{\infty} \right),
\end{equation}
thus, there exists $\gamma > 0$ such that
\begin{equation}
\| F_{\mathcal{R}_\gamma}(x) - G_{\mathcal{N}}(x) \| \leq \frac{\epsilon}{2}, \qquad \forall x \in C.
\end{equation}
Combining the two inequalities, we obtain
\begin{equation}
\| F_{\mathcal{R}}(x) - f(x) \| \leq \| F_{\mathcal{R}}(x) - G_{\mathcal{N}}(x) \| + \| G_{\mathcal{N}}(x) - f(x) \| \\
\leq \epsilon, \qquad \forall x \in C.
\end{equation}
Hence the result.
\end{proof}

\clearpage
\section{Bidirectional Amplifiers}
\label{sec:bidirectional-amplifiers}

One constraint of deep resistive networks is the decay in amplitude across the layers of the network, as the depth increases, and therefore the need to amplify the input voltages by a large gain $A^{(0)} \gg 1$ to compensate for this decay. Instead of amplifying input voltages by a large factor $A^{(0)}$, another solution proposed by Kendall \emph{et al.} \citep{kendall2020training} is to equip each `unit' (or node) with a \textit{bidirectional amplifier}.

A bidirectional amplifier is a three-terminal device, with bottom ($B$), left ($L$), and right ($R$) terminals. The bottom terminal is linked to ground. The current and voltage states $(v_L,i_L)$ and $(v_R,i_R)$ of the left and right terminals satisfy $v_R = a \, v_L$ and $i_R = a \, i_L$, for some positive constant $a$, the \textit{gain} of the bidirectional amplifier (see Figure~\ref{fig:bidirectional-amplifier}). Voltages are amplified in the forward direction by a factor $a>1$, and currents are amplified in the backward direction by a factor $1/a$. In practice, a bidirectional amplifier can be formed by assembling a voltage-controlled voltage source and a current-controlled current source.

We then equip each unit with a bidirectional amplifier, placed before the resistor and the diode (Figure~\ref{fig:bidirectional-amplifier}). We define the unit's state as the voltage after amplification. All the units of a given layer $\ell$ have the same gain $a^{(\ell)}$ as defined by \eqref{eq:layer-gain}. Lemma~\ref{lma:deep-resistive-network} can be restated as follows: for every $(\ell,k)$ such that $1 \leq \ell \leq L-1$ and $0 \leq k \leq N_\ell-1$, the equation for $v_k^{(\ell)}$ is
\begin{equation}
v_k^{(\ell)} = 
\left\{
\begin{array}{l}
    \displaystyle \max \left( 0, \frac{\sum_{j=0}^{N_{\ell-1}-1} g_{jk}^{(\ell)} a^{(\ell)} v_j^{(\ell-1)} + \sum_{j=2}^{N_{\ell+1}-1} g_{kj}^{(\ell+1)} \frac{v_j^{(\ell+1)}}{ a^{(\ell+1)}}}{g_k^{(\ell)} + \sum_{j=0}^{N_{\ell-1}-1} g_{jk}^{(\ell)} + \sum_{j=2}^{N_{\ell+1}-1} g_{kj}^{(\ell+1)}} \right) \qquad \text{if k is even (the unit is excitatory)}, \\
   \displaystyle \min \left( 0, \frac{\sum_{j=0}^{N_{\ell-1}-1} g_{jk}^{(\ell)} a^{(\ell)} v_j^{(\ell-1)} + \sum_{j=2}^{N_{\ell+1}-1} g_{kj}^{(\ell+1)} \frac{v_j^{(\ell+1)}}{ a^{(\ell+1)}}}{g_k^{(\ell)} + \sum_{j=0}^{N_{\ell-1}-1} g_{jk}^{(\ell)} + \sum_{j=2}^{N_{\ell+1}-1} g_{kj}^{(\ell+1)}} \right) \qquad \text{if k is odd (the unit is inhibitory)}.
\end{array}
\right.
\end{equation}
As for the equation of the output unit $v_k^{(L)}$,
\begin{equation}
v_k^{(L)} = 
\frac{\sum_{j=0}^{N_{L-1}-1} g_{jk}^{(L)} a^{(L)} v_j^{(L-1)}}{g_k^{(L)} + \sum_{j=0}^{N_{L-1}-1} g_{jk}^{(L)}}, \qquad 0 \leq k \leq N_L-1.
\end{equation}
In these expressions, $\frac{v_j^{(\ell+1)}}{a^{(\ell+1)}}$ represents the voltage of the $j$-th unit of layer $\ell+1$ \textit{before} amplification.

Lemma~\ref{lma:maximum-principle} must also be adapted: with bidirectional amplifiers, the bounds depend not only on the boundary conditions (input and bias voltages) but also on the layerwise gains of the bidirectional amplifiers ($a^{(\ell)}$). For simplicity, let us assume no bias voltages. Then the equation for $v_{\rm max}^{(\ell)} := \max_{2 \leq j \leq N_\ell-1} |v_j^{(\ell)}|$ takes the form
\begin{equation}
v_{\rm max}^{(\ell)} \leq a^{(\ell)} \, v_{\rm max}^{(\ell-1)}, \qquad 1 \leq \ell \leq L.
\end{equation}
Multiplying by $A^{(\ell)} = a^{(\ell+1)} \times \cdots \times a^{(L)}$ on both sides, we get
\begin{equation}
A^{(\ell)} v_{\rm max}^{(\ell)} \leq A^{(\ell-1)} \, v_{\rm max}^{(\ell-1)}, \qquad 1 \leq \ell \leq L.
\end{equation}
Iterating, we obtain
\begin{equation}
A^{(\ell)} | v_k^{(\ell)} | \leq  A^{(0)} \max_{2 \leq j \leq N_0-1} | v_j^{(0)} |, \qquad 0 \leq \ell \leq L, \quad 2 \leq k \leq N_\ell-1.
\end{equation}

\begin{figure}
\begin{center}
\fbox{
\includegraphics[width=0.4\textwidth]{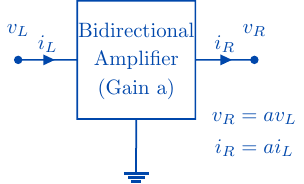}
}
\vspace{0.5cm}
\includegraphics[width=0.95\textwidth]{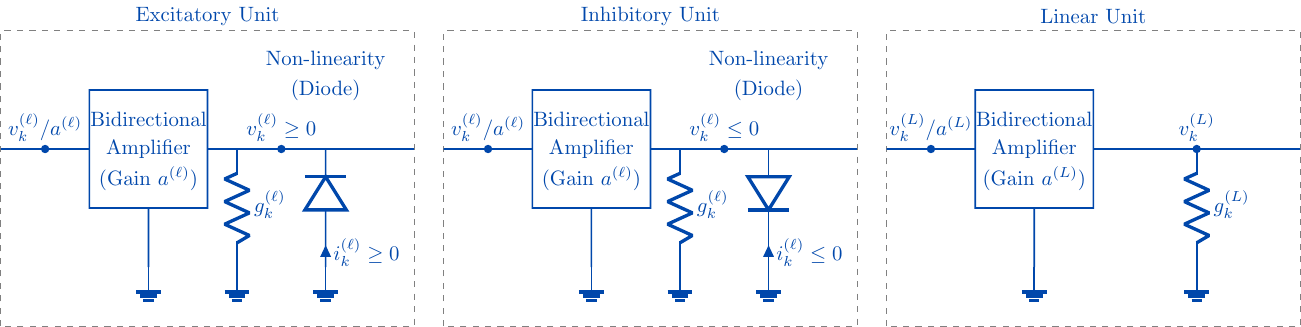}
\fbox{
\includegraphics[width=0.95\textwidth]{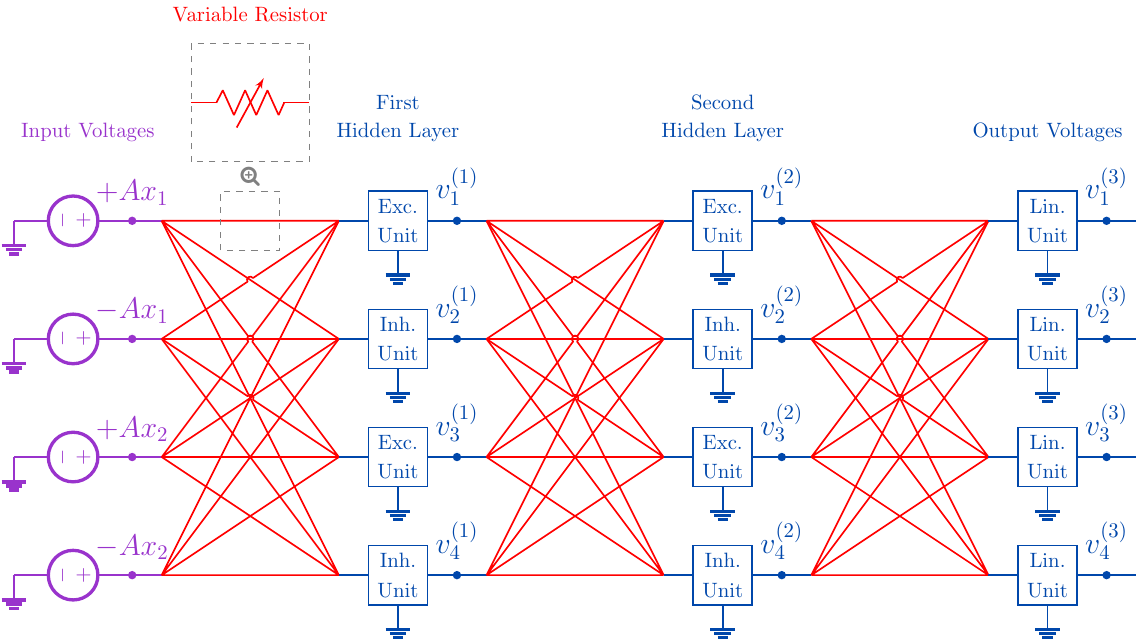}
}
\end{center}
\caption{
\textbf{Top.} Bidirectional amplifier with gain $a$. \textbf{Middle.} A unit is composed of a bidirectional amplifier (with amplification factor $a$), followed by a resistor and a diode between the unit's node and ground. Depending on the orientation of the diode, units come in two flavors: excitatory units and inhibitory units.
\textbf{Bottom.} A DRN with bidirectional amplifiers.
\label{fig:bidirectional-amplifier}
}
\end{figure}

\clearpage
\section{Simulation Details}
\label{sec:simulation-details}

In this appendix, we provide the details of our simulations (Section~\ref{sec:simulations}). We train DRNs with one, two and three hidden layers, and we compare their performances to their corresponding ReLU neural networks with one, two, and three hidden layers. Our code to reproduce the results is available at \url{https://github.com/rain-neuromorphics/energy-based-learning}

\paragraph{Datasets.}
We perform simulations on the MNIST, Kuzushiji-MNIST, and Fashion-MNIST datasets.

The MNIST dataset \citep{lecun1998gradient} consists of 60000 training images and 10000 test images. Each image is a 28x28 pixel grayscale representation of a handwritten digit ranging from 0 to 9. Each image is associated with a label indicating the digit it represents.

The Kuzushiji-MNIST dataset \citep{clanuwat2018deep} is a dropin replacement for the MNIST dataset, also consisting of 60000 training images and 10000 test images, each a 28x28 pixel grayscale image. The images represent handwritten kuzushiji (cursive Japanese) hiragana characters. Like MNIST, each image in the Kuzushiji-MNIST dataset comes with an associated label, corresponding to Japanese characters.

The Fashion-MNIST dataset \citep{xiao2017fashion} also features 60000 training images and 10000 test images. These images are also 28x28 pixels in grayscale, depicting various fashion items categorized into ten types such as shirts, trousers, and shoes.

\paragraph{Optimizer and scheduler.}
We optimize the mean squared error. We perform stochastic gradient descent with momentum $\mu=0.9$. No weight decay is used. We also use an `exponential learning rate scheduler' with a decay rate of $0.99$ for the learning rates at each epoch of training. We use mini batches of size 32 and we perform 100 epochs of training.

\paragraph{Computational resources.}
Our code for the simulations uses PyTorch 1.13.1 and TorchVision 0.14.1. \cite{paszke2017automatic}. The simulations were carried on three Nvidia A100 GPUs (one for each of the three sets of simulations corresponding to the three datasets). Each run took between 1 and 5 hours to complete, depending on the network type (DRN or ReLU neural network) and network architecture (1H, 2H, or 3H).

\subsection{Simulation Details of DRNs}

Table~\ref{table:hparams-drn} contains the hyperparameters used to obtain the results presented in Table~\ref{table:results}.

\paragraph{DRN architecture.}
We train deep resistive networks with one, two, and three hidden layers, termed DRN-1H, DRN-2H, and DRN-3H, respectively. Each DRN has 1568 input units ($2\times28\times28$), 10 output units and 1024 units per hidden layer. In our implementation for the simulations, we did not include resistors between the units' nodes and ground (parallel to the diodes).

\paragraph{Input amplification factor.} Input voltages are amplified by a (nontrainable) gain factor $A^{(0)}$. The value of $A^{(0)}$ is reported in Table~\ref{table:hparams-drn} for each DRN architecture (DRN-1H, DRN-2H, and DRN-3H).

\paragraph{Initialization of the conductances of the resistors.} Given two consecutive layers of indices $\ell$ and $\ell+1$, denoting by $N_\ell$ the number of units in layer $\ell$, we initialize the matrix $g^{(\ell+1)}$ of conductances between these two layers as
\begin{equation}
g_{jk}^{(\ell+1)} = \max \left( 0,h_{jk}^{(\ell+1)} \right), \qquad h_{jk}^{(\ell+1)} \sim \mathcal{U}(-c,+c), \qquad c := \sqrt{\frac{1}{N_\ell}}.
\end{equation}

\paragraph{Computing the steady state.}
We compute the steady state of the DRN with the following algorithm, introduced in \cite{scellier2024fast}. First we set the state of input units ($\ell=0$) to their input values, and we initialize the state of the hidden layers and output layer to zero. Then, for $\ell=1$ to $\ell=L-1$, we perform the following operations. For every $k$ such that $1 \leq k \leq N_\ell$, we update the state of unit $v_k^{(\ell)}$ as follows:
\begin{equation}
\label{eq:update-rule-hidden-drn}
v_k^{(\ell)} \leftarrow 
\left\{
\begin{array}{l}
\displaystyle \max \left( 0, p_k^{(\ell)} \right) \qquad \text{if k is even}, \\
\displaystyle \min \left( 0, p_k^{(\ell)} \right) \qquad \text{if k is odd},
\end{array}
\right.
\qquad \text{where} \qquad
p_k^{(\ell)} := \frac{\sum_{j=1}^{N_{\ell-1}} g_{jk}^{(\ell)} v_j^{(\ell-1)} + \sum_{j=1}^{N_{\ell+1}} g_{kj}^{(\ell+1)} v_j^{(\ell+1)}}{\sum_{j=1}^{N_{\ell-1}} g_{jk}^{(\ell)} + \sum_{j=1}^{N_{\ell+1}} g_{kj}^{(\ell+1)}}.
\end{equation}
For the output layer ($\ell=L$), the update rule for the states of the output units $v_k^{(L)}$ is
\begin{equation}
\label{eq:update-rule-output-drn}
v_k^{(L)} \leftarrow 
\frac{\sum_{j=1}^{N_{\ell-1}} g_{jk}^{(L)} v_j^{(L-1)}}{\sum_{j=1}^{N_{\ell-1}} g_{jk}^{(L)}}.
\end{equation}
Importantly, we may update all the units of a given layer simultaneously, rather than sequentially, thus taking advantage of the parallelism of GPUs. Updating all the layers once, from inputs to output, constitutes one `iteration'. We perform $T$ iterations of this process, until convergence to steady state. This algorithm is an instance of exact block coordinate descent.

\paragraph{Training procedure.} We train our networks by stochastic gradient descent. At each training step, we proceed as follows. First we pick a minibatch $x$ of examples in the training set, and their corresponding labels $y$. Then we perform $T$ iterations of the layer updates until convergence to the steady state. Once at steady state, we compute (or measure) the training error rate for the current minibatch, to monitor training. Next, we perform another $K$ iterations of the layer updates, and we compute the gradient of the cost function with respect to the weights (conductances) through the trajectory. This corresponds to a form of truncated backpropagation, where we perform $T+K$ steps in total, but we only backpropagate through the last $K$ steps. Then, we update all the conductances simultaneously, proportionally to their gradient. Finally, we clip all the negative conductance values to zero, to ensure that all the conductances remain non-negative.

We also performed simulations using equilibrium propagation for extracting the conductance gradients, in place of (truncated) backpropagation (see Appendix~\ref{sec:detailed-simulations}).

\begin{table}[ht!]
\caption{Hyperparameters used for initializing and training the deep resistive networks (DRNs) and reproducing the results of Table~\ref{table:results}. LR means learning rate.}
\label{table:hparams-drn}
\vskip 0.15in
\begin{center}
\begin{small}
\begin{sc}
\begin{tabular}{cccc}
\toprule
 & DRN-1H & DRN-2H & DRN-3H \\
\midrule
Input amplification factor ($A^{(0)}$) & 480 & 2000 & 4000 \\
No. of iterations at inference ($T$) & 4 & 5 & 6 \\
No. of iterations during training ($K$) & 4 & 5 & 6 \\
LR weight 1 \& bias 1 ($\eta_1$) & 0.006 & 0.002 & 0.005 \\
LR weight 2 \& bias 2 ($\eta_2$) & 0.006 & 0.006 & 0.02 \\
LR weight 3 \& bias 3 ($\eta_3$) &  & 0.005 & 0.08 \\
LR weight 4 \& bias 4 ($\eta_4$) &  &  & 0.005 \\
Nudging parameter ($\beta$), used in Table~\ref{table:complete-results} & 0.5 & 1.0 & 2.0 \\
\bottomrule
\end{tabular}
\end{sc}
\end{small}
\end{center}
\vskip -0.1in
\end{table}

\subsection{Simulation Details of ReLU Neural Networks}

As a baseline for our DRNs, we train ReLU neural networks with one, two, and three hidden layers, termed ReLU NN-1H, ReLU NN-2H and ReLU NN-3H, respectively.

\paragraph{Architecture.}
Each of the three ReLU neural network models has 784 input units ($28\times28$), 10 output units, and 512 units per hidden layer. 

\paragraph{Initialization of the weights.}
We use the `Kaiming uniform' weight initialization scheme. Given two consecutive layers of indices $\ell$ and $\ell+1$, denoting by $N_\ell$ the number of units in layer $\ell$, we initialize the matrix of conductances between these two layers as
\begin{equation}
w_{jk}^{(\ell+1)} \sim \mathcal{U}(-c,+c), \qquad c := \sqrt{\frac{1}{N_\ell}}.
\end{equation}

\paragraph{Computing the weight gradients.}
At each step of stochastic gradient descent, we calculate the weight gradients  using backpropagation. We use $\eta=0.005$ as a learning rate for all the weights and biases.

\subsection{Extended Simulation Results}
\label{sec:detailed-simulations}

Table~\ref{table:results} presents the results we obtained using (truncated) backpropagation (TBP) for training DRNs. Table~\ref{table:complete-results} extends the results of Table~\ref{table:results} to include the results we obtained using equilibrium propagation (EP) for training. Table~\ref{table:complete-results} also provides the error rates on the training set (training error rate), in addition to the test error rate.

In all cases, DRNs trained with TBP achieve comparable test performance to their equivalent-size ReLU neural networks. On the other hand, we observe that the training error rate in DRNs is often higher (sometimes significantly higher) than in ReLU neural networks, hinting at an optimization issue for DRNs. We note, however, that our Theorem~\ref{thm:drn-approximate-relu-nn} only states the existence of a configuration of conductance values that approximates the behavior of a ReLU neural network at inference, but it does not provide any guarantee that this configuration of conductance values can easily be reached through gradient descent optimization. We also observe that DRNs trained with EP often achieve lower performance than with TBP. Again, understanding these differences in performance is beyond the scope of our work. In this work, we provide a proof of existence, and we leave these open questions for future works to investigate.

\begin{table}[h!]
\caption{This table provides a more extensive presentation of the results of Table~\ref{table:results}. We compare the performance of DRNs trained with equilibrium propagation (EP), DRNs trained with truncated backpropagation (TBP), and their equivalent-size ReLU neural networks trained with backpropagation (BP). We train DRNs and ReLU neural networks with one, two, and three hidden layers on MNIST, Kuzushiji-MNIST and Fashion-MNIST. For each setting, we perform five runs and report the mean and std values (\%) of the training and test error rates.}
\label{table:complete-results}
\vskip 0.15in
\begin{center}
\begin{small}
\begin{sc}
\begin{tabular}{cccc}
    \toprule
    \multirow{2}{*}{Dataset} & Network Model & \multirow{2}{*}{Test Error (\%)} & \multirow{2}{*}{Train Error (\%)} \\
     & (Learning Algorithm) & & \\
    \midrule
    & DRN-1H (EP) & 1.53 $\pm$ 0.05 & 0.07 $\pm$ 0.00 \\
    & DRN-1H (TBP) & 1.51 $\pm$ 0.04 & 0.04 $\pm$ 0.00 \\
    & ReLU NN-1H (BP) & 1.67 $\pm$ 0.03 & 0.22 $\pm$ 0.01 \\
    \cmidrule(r){2-4}
    \multirow{3}{*}{MNIST} & DRN-2H (EP) & 1.60 $\pm$ 0.05 & 0.20 $\pm$ 0.01 \\
    & DRN-2H (TBP) & 1.46 $\pm$ 0.05 & 0.20 $\pm$ 0.01 \\
    & ReLU NN-2H (BP) & 1.34 $\pm$ 0.01 & 0.03 $\pm$ 0.01 \\
    \cmidrule(r){2-4}
    & DRN-3H (EP) & 1.77 $\pm$ 0.08 & 0.36 $\pm$ 0.01 \\
    & DRN-3H (TBP) & 1.47 $\pm$ 0.08 & 0.30 $\pm$ 0.01 \\
    & ReLU NN-3H (BP) & 1.48 $\pm$ 0.03 & 0.02 $\pm$ 0.00 \\
    \midrule
    & DRN-1H (EP) & 7.57 $\pm$ 0.11 & 0.08 $\pm$ 0.01 \\
    & DRN-1H (TBP) & 7.67 $\pm$ 0.13 & 0.06 $\pm$ 0.00 \\
    & ReLU NN-1H (BP) & 8.83 $\pm$ 0.11 & 0.12 $\pm$ 0.01 \\
    \cmidrule(r){2-4}
    \multirow{2}{*}{Kuzushiji} & DRN-2H (EP) & 8.40 $\pm$ 0.10 & 0.39 $\pm$ 0.03 \\
    \multirow{2}{*}{MNIST} & DRN-2H (TBP) & 8.27 $\pm$ 0.02 & 0.42 $\pm$ 0.01 \\
    & ReLU NN-2H (BP) & 8.14 $\pm$ 0.13 & 0.05 $\pm$ 0.01 \\
    \cmidrule(r){2-4}
    & DRN-3H (EP) & 9.30 $\pm$ 0.17 & 0.63 $\pm$ 0.03 \\
    & DRN-3H (TBP) & 8.39 $\pm$ 0.24 & 0.63 $\pm$ 0.01 \\
    & ReLU NN-3H (BP) & 7.99 $\pm$ 0.23 & 0.03 $\pm$ 0.01 \\
    \midrule
    & DRN-1H (EP) & 10.38 $\pm$ 0.11 & 6.14 $\pm$ 0.15 \\
    & DRN-1H (TBP) & 10.31 $\pm$ 0.05 & 5.91 $\pm$ 0.19 \\
    & ReLU NN-1H (BP) & 10.20 $\pm$ 0.11 & 5.13 $\pm$ 0.01 \\
    \cmidrule(r){2-4}
    \multirow{2}{*}{Fashion} & DRN-2H (EP) & 10.29 $\pm$ 0.14 & 6.20 $\pm$ 0.06 \\
    \multirow{2}{*}{MNIST} & DRN-2H (TBP) & 9.88 $\pm$ 0.11 & 4.88 $\pm$ 0.05 \\
    & ReLU NN-2H (BP) & 9.42 $\pm$ 0.22 & 2.78 $\pm$ 0.03 \\
    \cmidrule(r){2-4}
    & DRN-3H (EP) & 11.23 $\pm$ 0.17 & 8.19 $\pm$ 0.05 \\
    & DRN-3H (TBP) & 9.79 $\pm$ 0.25 & 4.58 $\pm$ 0.04 \\
    & ReLU NN-3H (BP) & 9.55 $\pm$ 0.14 & 1.82 $\pm$ 0.05 \\
    \bottomrule
\end{tabular}
\end{sc}
\end{small}
\end{center}
\vskip -0.1in
\end{table}

\clearpage
\section{Equilibrium Propagation}
\label{sec:equilibrium-propagation}

The deep resistive networks studied in this work were introduced within the equilibrium propagation training framework \citep{scellier2017equilibrium,kendall2020training}. This appendix highlights the advantages that DRNs and EP offer over analog electrical implementations of neural networks using backpropagation. We begin by discussing the challenges associated with analog implementations of BP.

\subsection{Analog Backpropagation}

There has been a long-standing effort to develop analog versions of neural networks and BP. Most electrical implementations employ crossbar arrays of memristors to perform matrix-vector multiplications (MVMs) \citep{xiao2020analog}, which constitute the core arithmetic operations of neural networks and BP. However, two additional operations are required: computing the nonlinear activation function, and computing its derivative. Based on whether these operations are executed in the analog or digital domain, analog implementations of BP can be broadly divided into two categories.

The first category aims to implement BP entirely in the analog domain, not just the MVMs. For instance, \cite{hermans2015trainable} utilized voltage followers to implement the ReLU activation function during the forward pass and analog switches to compute its derivative during the backward pass. However, this approach faces a significant challenge: analog devices cannot perfectly replicate the desired operations, leading to a mismatch between the activation function and its derivative. In practice, this mismatch will cause discrepancies between the calculated and actual weight gradients, leading to suboptimal optimization and reduced accuracy.

The second category employs mixed-signal (analog/digital) circuits, where MVMs are performed in the analog domain while nonlinear activation functions and their derivatives are computed digitally \citep{xiao2020analog}. Such mixed-signal implementations of neural networks can mitigate the mismatch issue ; however, the energy gains achieved by performing MVMs in the analog domain are offset by the power demands of analog-to-digital conversion required before the activation function \citep{li2015merging}.

Thus, it remains unclear whether BP is the most suitable training algorithm for analog computing platforms. This has led to the exploration of alternative approaches, such as EP, which may be better suited for these systems.

\subsection{Equilibrium Propagation}

Equilibrium propagation \citep{scellier2017equilibrium} has emerged as an alternative to BP for training analog systems, including nonlinear resistive networks \citep{kendall2020training}. For a more general, physics-oriented overview of EP, see \cite{scellier2021deep}. Here, we focus on EP's application in DRNs, and explain how it addresses the mismatch problem found in analog BP.

EP is broadly applicable to energy-based systems that minimize an `energy' (or Lyapunov) function. For a DRN with ideal components, the energy function corresponds to power dissipation:
\begin{equation}
E_{\rm DRN}^{\rm ideal}(v) = \sum_{\ell=1}^L \sum_{j=1}^{N_{\ell}} \sum_{k=1}^{N_{\ell+1}}  g_{jk}^{(\ell)} \left( v_j^{(\ell)}-v_k^{(\ell+1)} \right)^2,
\end{equation}
where $g_{jk}^{(\ell)}$ represents a memristor conductance, $v_j^{(\ell)}$ denotes a node electrical potential, and $v$ is the vector of all node potentials. As shown by \cite{scellier2024fast}, these ideal DRNs reach a steady state $v_\star$, given by
\begin{equation}
v_\star = \underset{v \in \mathcal{S}_{\rm DRN}^{\rm ideal}}{\arg \min} \; E_{\rm DRN}^{\rm ideal}(v),
\end{equation}
where the feasible set $\mathcal{S}_{\rm DRN}^{\rm ideal}$ is defined by diode constraints
\begin{equation}
\mathcal{S}_{\rm DRN}^{\rm ideal} = \{ v \in \mathbb{R}^{\sum_{\ell=1}^L N_\ell} \mid v_k^{(\ell)} \geq 0 \text{ if k is even}, v_k^{(\ell)} \leq 0 \text{ if k is odd}, 1 \leq \ell \leq L-1, 1 \leq k \leq N_\ell \}.
\end{equation}
In the present work, we have assumed ideal diodes to simplify the mathematical analysis of DRNs and derive our universal approximation theorem (specifically, this assumption has enabled us to derive Lemma~\ref{lma:deep-resistive-network} and subsequently an approximate equivalence with ReLU neural networks, Theorem~\ref{thm:drn-approximate-relu-nn}). Similarly, \cite{scellier2024fast} assumed ideal diodes to derive a fast algorithm to simulate DRNs on digital computers. Importantly, however, EP does not rely on this assumption of ideality. In fact, the $i$-$v$ characteristics of nonlinear untrainable elements (diodes) do not need to be explicitly known. Indeed, nonlinear resistive networks with nonideal elements possess an energy function called the `co-content' \citep{millar1951cxvi} or `pseudo power' \citep{kendall2020training}. For instance, the co-content for a DRN with linear resistors and nonideal diodes is of the form
\begin{equation}
E_{\rm DRN}(v) = \sum_{\ell=1}^L \sum_{j=1}^{N_{\ell}} \sum_{k=1}^{N_{\ell+1}}  g_{jk}^{(\ell)} \left( v_j^{(\ell)}-v_k^{(\ell+1)} \right)^2 + E_{\rm diodes}(v),
\end{equation}
where the term $E_{\rm diodes}$ depends solely on the $i$-$v$ characteristics of the diodes. The steady state is then
\begin{equation}
v_\star = \arg \min_v E_{\rm DRN}(v).
\end{equation}
The goal of learning is to minimize a cost function $C(v_\star)$ with respect to the trainable weights (memristor conductances), given boundary input conditions from voltage sources. For example, using the mean squared error,
\begin{equation}
C(v) = \sum_k \left( v_k^{\rm out} - y_k \right)^2,
\end{equation}
where $v_k^{\rm out}$ is the $k$-th output node's potential, and $y_k$ is the $k$-th target output.

The core idea of EP is to interpret $\beta C$ as power dissipation by resistors with conductance $\beta$ linking output nodes to `target nodes' (whose electrical potentials are set to target values). The resulting energy function (co-content) is
\begin{equation}
F(v,\beta) = E_{\rm DRN}(v) + \beta C(v),
\end{equation}
where $\beta \geq 0$ is termed the `nudging parameter'. Each training step of EP consists of:
\begin{enumerate}
\item sourcing input voltages, setting $\beta=0$ and allowing the system to reach the `free state', $v_\star^0 = \arg \min_v F(0,v)$ ;
\item increasing $\beta>0$, and allowing the system to reach the `nudge state', $v_\star^\beta = \arg \min_v F(\beta,v)$.
\end{enumerate}
The gradient of the cost function with respect to the conductances can be estimated using the formula \citep{scellier2017equilibrium}
\begin{equation}
\frac{d}{dg} C \left( v_\star \right) = \left. \frac{d}{d\beta} \right|_{\beta=0} \frac{\partial F}{\partial g} \left( v_\star^\beta,\beta \right) \approx \underbrace{\frac{1}{\beta} \left( \frac{\partial F}{\partial g} \left( v_\star^\beta,\beta \right) - \frac{\partial F}{\partial g} \left( v_\star^0,0 \right) \right)}_{=: \widehat{\nabla}_g(\beta)}.
\end{equation}
Using the analytical form of the energy function, we obtain
\begin{equation}
\widehat{\nabla}_{g_{jk}^{(\ell)}}(\beta) = \frac{1}{\beta} \left( \frac{\partial F}{\partial g_{jk}^{(\ell)}} \left( v_\star^\beta,\beta \right) - \frac{\partial F}{\partial g_{jk}^{(\ell)}} \left( v_\star^0,0 \right) \right) = \frac{1}{\beta} \left[ \left( (v_j^{(\ell)})^\beta - (v_k^{(\ell+1)})^\beta \right)^2 - \left( (v_j^{(\ell)})^0 - (v_k^{(\ell+1)})^0 \right)^2 \right].
\end{equation}
Finally, we update the weights (memristor conductances) using these gradient estimates, to perform (or approximate) one step of gradient descent on the cost function:
\begin{equation}
\Delta g = - \eta \widehat{\nabla}_g(\beta).
\end{equation}
The learning rule for each weight (memristor) is local, and notably, it is agnostic to the $i$-$v$ characteristics of the diodes. Whether the network components are ideal or nonideal, the EP algorithm and its learning rules remain unchanged. This feature makes EP robust to variability of analog devices, particularly untrainable nonlinear elements (diodes), which is a key advantage of DRNs and EP over analog implementations of neural networks and BP. Intuitively, the robustness of EP arises because the same network is used for both phases of training (free and nudge), so that nonidealities affect both states similarly and thus cancel out in the gradient calculation. EP thus solves the mismatch problem encountered in BP between the nonlinear activation function and its derivative.

In addition to the resistive networks considered in this work, EP has been applied to various energy-based systems, such as continuous Hopfield networks \citep{scellier2017equilibrium}, Ising machines \citep{laydevant2024training} and coupled phase oscillators \citep{wang2024training}. A variant of EP called `coupled learning' has been developed and used in other systems as well, such as flow and elastic networks \citep{stern2021supervised,altman2024experimental}. Recent works have extended EP to quantum systems too \citep{massar2025equilibrium,wanjura2024quantum,scellier2024quantum}.

While EP's learning rule does not require knowledge of the characteristics of untrainable elements, it does require knowledge of the partial derivatives of the energy function with respect to trainable weights ($\frac{\partial E}{\partial g}$). However, a more advanced version of EP called `Agnostic Equilibrium Propagation' has been introduced, which does not require knowledge of these partial derivatives either, and allows the weight updates to be carried out using physical dynamics \citep{scellier2022agnostic}.

\end{document}